\documentclass{article}

 \usepackage[margin=1.2in]{geometry}

\usepackage[utf8]{inputenc}
\usepackage[T1]{fontenc}
\usepackage[english]{babel}
\usepackage{subcaption}
\usepackage{graphicx}
\usepackage{booktabs}
\usepackage{nicefrac}
\usepackage{microtype}
\usepackage{xcolor,colortbl}
\usepackage{amsmath}
\usepackage{amsthm}
\usepackage{amssymb}
\usepackage{tikz}
\usetikzlibrary{arrows.meta,decorations.pathreplacing,positioning}
\usepackage{faktor}
\usepackage{makecell}
\usepackage{comment}
\usepackage{wasysym}
\usepackage{float}
\usepackage{mathrsfs}
\usepackage{url}
\usepackage{esint}
\usepackage{bbm}
\usepackage{bm}
\usepackage{hyperref}
\usepackage{cleveref}
\usepackage{etoc}

\hypersetup{
colorlinks = true,
urlcolor   = blue,
citecolor  = blue,
linkcolor  = blue,
}

\DeclareMathOperator{\diver}{div}

\theoremstyle{plain}
\newtheorem{thm}{Theorem}[section]

\newtheorem*{thm*}{Theorem}
\newtheorem{lem}[thm]{Lemma}
\newtheorem{prop}[thm]{Proposition}

\theoremstyle{remark}
\newtheorem{rmk}[thm]{Remark}

\newcommand{\R}{\mathbb{R}}

\newcommand{\T}{\mathbb{T}}

\newcommand{\Z}{\mathbb{Z}}

\newcommand{\eps}{\varepsilon}

\numberwithin{equation}{section}

\usepackage{authblk}

\title{Quantitative Local Convergence of \\ Mean-Field Stein Variational Gradient Flow}
\author{  L\'ena\"ic Chizat, Maria Colombo, Roberto Colombo, Xavier Fern\'andez-Real}

\date{}

\begin{document}
\maketitle

\begingroup
\makeatletter
\renewcommand{\thefootnote}{$\ast$}
\renewcommand{\@makefnmark}{\normalfont\@thefnmark}
\footnotetext{\, Authors are listed in alphabetical order.}
\makeatother
\endgroup

\vspace{-2em}

\begin{abstract}
\vspace{0.5em}
Stein Variational Gradient Descent (SVGD) is a deterministic interacting-particle method for sampling from a target probability measure given access to its score function.
In the mean-field and continuous-time limit, it is known that the flow converges weakly toward the target, but no quantitative rate is known for the last iterate.
In this paper, we establish quantitative local  convergence in strong norms for this dynamics, when the interaction kernel is of Riesz type on the $d$-dimensional torus.
Specifically, assuming that the initial density and the target are smooth and close in $L^2$-norm,  we obtain explicit polynomial convergence rates in $L^2$-norm that depend on the dimension and on the regularity parameters of the kernel, the initialization and the target. 
We further show that these rates are sharp in certain regimes, and support the theory with numerical experiments. 
In the edge case of kernels with a Coulomb singularity, we recover the global exponential convergence result established in prior work.
Our analysis is inspired by recent results on Wasserstein gradient flows of kernel mean discrepancies.
\end{abstract}


%\tableofcontents

\section{Introduction}
Stein Variational Gradient Descent (SVGD), introduced in~\cite{liu2016stein}, is a deterministic interacting-particle method for sampling from a target probability measure $\pi\propto e^{-V}$, only requiring access to $\nabla V$.
In the mean-field and continuous-time limit, the distribution of particles converges to a flow $(\rho_t)$ in the space of probability measures that solves a variant of the Fokker-Planck equation with a velocity field smoothed by weighted convolution with a positive definite kernel~\cite{lu2019scaling}. This flow can be interpreted as the gradient flow of the relative entropy $\mathcal{H}(\cdot|\pi)$ with respect to a ``kernelized'' Wasserstein metric~\cite{liu2017stein}. 

The goal of this paper is to investigate the convergence of $(\rho_t)$ towards $\pi$. To this end, we focus on the model case of Riesz kernels of order $s$ on the $d$-dimensional torus $\mathbb{T}^d$. This is a family of translation-invariant kernels whose Fourier coefficients decay as $|\xi|^{-2s}$. The parameter $s$ hence directly controls the ``smoothing strength'' of the interaction; in particular, continuous kernels correspond to $s>d/2$, $C^1$ kernels to $s>(d+1)/2$, and $C^2$ kernels to $s>(d+2)/2$. 

\paragraph{What is known: qualitative weak convergence} The starting point of convergence analyses is the energy dissipation formula~\cite{liu2017stein}
\begin{equation}\label{eq:intro-dissipation}
 \frac{d}{dt}\mathcal{H}(\rho_{t}|\pi)=-I_{s}(\rho_{t}|\pi),
\end{equation}
where $I_s\geq 0$ is the so-called kernel Stein discrepancy (KSD), see~\eqref{eq:stein-dissipation-identity}. 
 In our setting, $I_s(\rho|\pi)$ is comparable to the squared negative Sobolev norm $\Vert \rho-\pi\Vert^2_{{H}^{1-s}}$ (see Lemma~\ref{lem:comparability-stein-H1-s}), which vanishes iff $\rho=\pi$, so $\pi$ is the only equilibrium of the dynamics. Since the flow $(\rho_t)$ is also weakly compact, one may expect that $\rho_t$ converges weakly to $\pi$, which was indeed shown in~\cite[Prop.~2]{korba2020non} for smooth potentials and kernels.  In Proposition~\ref{prop:qualitative-convergence-SVGD}, we give an alternative proof of this qualitative weak convergence under slightly weaker regularity assumptions.

\paragraph{What is not known: quantitative strong convergence} Such a result gives, however, limited information about the finite-time behavior. For instance, one cannot exclude a priori that $I_s(\rho_t|\pi)$ remains small for an arbitrarily long time and then increases significantly.
Several works overcome this difficulty by considering the averaged iterate $\bar \rho_t=\frac1t \int_0^t \rho_{u}du$ 
which, by~\eqref{eq:intro-dissipation} and convexity of $I_s(\cdot|\pi)$, satisfies the quantitative rate
\begin{align}\label{eq:basic-rate}
I_s(\bar \rho_t|\pi)\leq \mathcal{H}(\rho_0|\pi)/t.
\end{align}
The best iterate ${\rho}^*_t$ over $[0,t]$ also satisfies~\eqref{eq:basic-rate}. However, $\bar \rho_t$ and $\rho^*_t$  are impractical to compute and rarely used. Another limitation of this class of results is that $I_s$ is a weak notion of discrepancy, especially in high dimension, and it depends on the choice of kernel, making it difficult to compare the behavior for various kernels. Moreover, there is no indication that the rate $O(1/t)$ is tight.  We therefore ask the following question:

\begin{center}
\emph{Can one obtain sharp quantitative convergence rates for the mean-field SVGD flow in strong norms?}
\end{center}

\paragraph{Main results}
Our main contribution (Theorem~\ref{thm:quantitative-convergence-s>1}) answers this question in a local sense. Namely, for all $s>1$, assuming that the initial density and the target are smooth and close to each other in $L^2$-norm,  we obtain explicit polynomial convergence rates in $L^2$-norm that depend on $s$, $d$, and on the regularity parameters of the initialization and the target. The polynomial rate in Theorem~\ref{thm:quantitative-convergence-s>1} is sharp when the target $\pi$ is uniform. We give in Remark~\ref{rmk:conjecture} a slightly improved rate for general targets, which we conjecture to be sharp in all regimes where it applies.

The edge case $s=1$ corresponds to kernels with a Coulomb singularity on the diagonal\footnote{Namely, the kernel explodes near the diagonal as $\mathrm{dist}(x,y)^{-d+2}$  when $d\geq 3$.}. In this case, both the energy and the energy dissipation are comparable to $\Vert \rho_t -\pi\Vert^2_{L^2}$, so by~\eqref{eq:intro-dissipation} and Gr\"onwall's lemma, one deduces global convergence in $L^2$ at an exponential rate. This result was already established in the inspiring work~\cite{carrillo2024stein} for kernels with a Coulomb singularity and suitable tails on $\mathbb{R}^d$. In Theorem~\ref{thm:quantitative-convergence-s=1}, we recover a compact analogue of their result with a short proof, and in addition, we show the global well-posedness of the flow in $L^\infty$ via a maximum principle. 

Note, however, that the practical implementation of SVGD requires, at least, that the gradient of the kernel is locally bounded, which holds iff $s\ge (d+1)/2$. So in dimension $d\geq 2$, the practically relevant regime is that of Theorem~\ref{thm:quantitative-convergence-s>1}, where the energy dissipation involves a norm weaker than the energy itself, preventing a global \L{}ojasiewicz-type inequality. We overcome this difficulty by showing refined propagation of regularity estimates; see Section~\ref{subsec:idea-proof} for the proof idea. Our analysis is inspired by recent results on Wasserstein gradient flows of kernel mean discrepancies~\cite{chizat2026quantitative}.

From a practical viewpoint, Theorem~\ref{thm:quantitative-convergence-s>1} sheds a new light on the behavior of SVGD. It is the first result that yields quantitative guarantees depending (sharply) on properties of the kernel and the target\footnote{This is, to a certain extent, analogous to the rates expressed under source/capacity conditions in the literature on kernel regression with gradient descent~\cite{yao2007early}. A major difference however is that the dynamical system we study is non-linear.}. In particular, the rates from Remark~\ref{rmk:conjecture} suggest that a good choice of kernel should have $s$ as small as possible while maintaining stability of the discrete scheme, e.g.~$s=(d+1)/2$. This corresponds to kernels behaving like $-\mathrm{dist}(x,y)$ near the diagonal. In particular, Gaussian-like kernels, for which the Fourier coefficients decay super-polynomially, lead to very slow convergence rates in $L^2$, because higher frequency components of $\rho_t-\pi$ dissipate extremely slowly.

\paragraph{Related work} The influential works~\cite{liu2016stein, liu2017stein} introduced SVGD  along with its mean-field limit, energy decay formula~\eqref{eq:intro-dissipation} and gradient flow interpretation. The rigorous mean-field limit was studied in~\cite{lu2019scaling} and several works have investigated its properties~\cite{duncan2023geometry, nusken2023stein,chewi2020svgd}. Other works have focused on the discrete-time dynamics~\cite{salim2022convergence, sun2022note} and/or the particle dynamics~\cite{carrillo2025convergence, banerjee2024improved, he2026finite}. In those settings, the state-of-the-art analyses rely on ~\eqref{eq:basic-rate} or its discrete variants.  We note also that variants of SVGD with noise have been proposed in~\cite{priser2024long, gallego2018stochastic}: the noise tends to facilitate convergence because it makes the dynamics more similar to the overdamped Langevin dynamics, for which unconditional exponential convergence holds. Finally, some undesirable behaviors for SVGD in high dimension are known, such as a ``mode collapse'' phenomenon~\cite{ba2021understanding}. Our work highlights another problem: unless the regularity of the target increases with the dimension, the convergence guarantee in $L^2$ degrades exponentially with the dimension, by the conjectured sharpness of the rate in Remark~\ref{rmk:conjecture}. 

\paragraph{Organization} In Section~\ref{sec:presentation-results}, we detail our setting, state the main results, and discuss their implications. Numerical experiments in synthetic settings are in Section~\ref{sec:numerics}, and Section~\ref{sec:proof} contains the proof, with some  lemmas postponed to the supplementary materials. We conclude in Section~\ref{sec:conclusion}.

\section{Setting and main results}
\label{sec:presentation-results}

\subsection{Setting}\label{sec:setting}
Let $d\ge 1$, and let $\T^{d}$ be the $d$-dimensional torus. For any $s\ge 1$, we denote by $K_{s}$ the Riesz kernel of order $s$ solving $(-\Delta)^{s}K_{s}=\delta_{0}-1$ in the sense of distributions. We refer to~\cite[Appendix~A.1]{chizat2026quantitative} for a detailed description of $K_s$.

Let $V:\T^{d}\to \R$ be a sufficiently regular potential, and consider the associated probability measure 
\begin{equation}\label{eq:def-pi-intermsof-V}
    \pi=\frac{e^{-V}}{Z}\in \mathcal{P}(\T^{d}),\qquad Z:= \int_{\T^{d}}e^{-V}.
\end{equation}
We study solutions to the following equation, which we refer to as the Mean-Field Stein Variational Gradient Flow (mean-field SVGF):
\begin{equation}\label{eq:SVGD}
     \left\{
    \begin{array}{rclll}
         \partial_{t}\rho+\diver(\rho v)&=&0\qquad &\text{in $(0,T)\times \T^{d}$},\\
            v_{t}&=&-K_{s}*\left(\rho_{t}\nabla\log\left(\frac{\rho_{t}}{\pi}\right)\right)\qquad &\forall t\in (0,T),\\
            \rho_{0}&=&\bar\rho.
    \end{array}
    \right.
\end{equation}
Let $\mathcal{H}(\rho|\pi):= \int_{\T^{d}}\log\left(\frac{\rho}{\pi}\right)\rho$ denote relative entropy of $\rho$ with respect to $\pi$. Then, a solution $\rho_{t}$ of \eqref{eq:SVGD} satisfies formally the following dissipation identity:
\begin{equation}\label{eq:stein-dissipation-identity}
    \frac{d}{dt}\mathcal{H}(\rho_{t}|\pi)=-\int_{\T^{d}}\int_{\T^{d}}K_{s}(x-y)\nabla\frac{\rho_{t}}{\pi}(x)\cdot \nabla \frac{\rho_{t}}{\pi}(y)\pi(x)\pi(y)=-I_{s}(\rho_{t}|\pi),
\end{equation}
where the last equality defines $I_{s}(\rho|\pi)$, known as the kernel Stein discrepancy (KSD) or as the Stein relative Fisher information. As detailed in Lemma~\ref{lem:comparability-stein-H1-s}, it holds $I_s(\rho|\pi)\approx \Vert \rho -\pi\Vert^2_{\dot{H}^{1-s}}$ and therefore $I_s$ vanishes iff $\rho=\pi$. In particular, $\pi$ is the unique equilibrium point of the dynamics. 

We begin with the qualitative weak convergence of the dynamics, which is an adaptation of~\cite[Prop.~2]{korba2020non} to our setting. While~\cite{korba2020non} obtains this result by differentiating twice $\mathcal{H}(\cdot|\pi)$ along the flow, and therefore requires the potential $V$ and the kernel $K$ of regularity $C^2$, we use a different proof technique, exploiting the gradient flow structure, which requires less regularity.

\begin{prop}\label{prop:qualitative-convergence-SVGD}
    Let $s>\tfrac{d}{2}$, $V\in C^{1}(\T^{d})$, and let $\rho \in C_{t}\mathcal{P}_{x}$ be a global solution of \eqref{eq:SVGD} satisfying the energy dissipation inequality
    \begin{equation*}
        \mathcal{H}(\rho_{T}|\pi)+\int_{0}^{T}I_{s}(\rho_{t}|\pi)\,dt \le \mathcal{H}(\bar\rho| \pi)\qquad \forall T\ge 0.
    \end{equation*}
    Then $\rho_{t}$ converges weakly to $\pi$ as $t\to \infty$.
\end{prop}

\subsection{Statement of the main results}\label{sec:statement}

Our main result is the following quantitative local convergence in the case $s>1$:
\begin{thm}[Local polynomial convergence for $s>1$]\label{thm:quantitative-convergence-s>1}
    Let $s>1$ and $\gamma>\max\{\tfrac{d}{2},s-1\}$. Let $V\in H^{\gamma+s}(\T^{d})$ and $\pi\in \mathcal{P}\cap H^{\gamma+s}(\T^{d})$ be defined as in \eqref{eq:def-pi-intermsof-V}. Let $\bar\rho\in \mathcal{P}\cap  H^{\gamma}(\T^{d})$, and let $\rho\in C_{t}\mathcal{P}_{x}\cap L^{\infty}_{t}H^{\gamma}_{x}$ be a solution of \eqref{eq:SVGD} extended up to the maximal time of existence $T>0$. There are constants $C>0$ and $\delta>0$ depending only on $d,s,\gamma$, $\lVert V\rVert_{H^{\gamma+s}}$, and $\lVert \bar\rho\rVert_{H^{\gamma}}$ such that, if $\mathcal{H}(\bar\rho|\pi)\le \delta$, then $T=\infty$, and 
    \begin{equation}\label{eq:polynomial-decay-s>1}
        \mathcal{H}(\rho_{t}|\pi)\le \left(\mathcal{H}(\bar\rho|\pi)^{-\frac{s-1}{\gamma}}+\frac{t}{C}\right)^{-\frac{\gamma}{s-1}}\qquad \forall t\ge 0.
    \end{equation}
    Moreover, for every $\beta\in [0,\gamma]$ we have 
    \begin{equation}\label{eq:interpolation-decay-s>1}
        \lVert \rho_{t}-\pi\rVert^2_{\dot H^{\beta}}\lesssim \lVert \bar\rho-\pi\rVert_{\dot H^{\gamma}}^{2\frac{\beta}{\gamma}} \left(\mathcal{H}(\bar\rho|\pi)^{-\frac{s-1}{\gamma}}+\frac{t}{C}\right)^{-\frac{\gamma-\beta}{s-1}}\qquad \forall t\ge 0.
    \end{equation}
\end{thm}

\begin{rmk}[Sharpness of the rate for uniform targets]\label{rmk:uniform-target}
Under a control of the initial datum in $H^\gamma$, the exponent $-\gamma/(s-1)$ is already sharp for uniform targets. Indeed, as in~\cite[Rmk.~1.6]{chizat2026quantitative}, the zero-mean quantity $\sigma_{t}$ that solves the linearized equation
\[
    \partial_t\sigma_t=-(-\Delta)^{1-s}\sigma_t
\]
has Fourier coefficients evolving as
\[
    \widehat\sigma_k(t)=\widehat\sigma_k(0)
    \exp\!\left(-|2\pi k|^{2-2s}t\right).
\]
Taking $\sigma_0^n(x)=\sqrt{2}(2\mathrm{\pi}n)^{-\gamma}\cos(2\mathrm{\pi}n x_1)$ gives
$\|\sigma_0^n\|_{\dot H^\gamma}=1$ and, at time $t=n^{2s-2}$,
\[
    \|\sigma_t^n\|_{L^2}^2
    \gtrsim n^{-2\gamma}
    \simeq t^{-\frac{\gamma}{s-1}} .
\]
Thus neither the squared $L^2$-norm nor, by Lemma~\ref{lem:comparability-entropy-L2}, the relative entropy can decay faster.
\end{rmk}

\begin{rmk}[Conjecture for tight rates]\label{rmk:conjecture} 
By arguing as in~\cite[Remark~4.8]{chizat2026quantitative}, one can replace the assumption $V\in H^{\gamma+s}$ in Theorem~\ref{thm:quantitative-convergence-s>1} by $V\in H^{\gamma+\varepsilon}$ for some $\varepsilon>0$, at least when $\gamma$ and $s$ are integers (we expect the same to be true for fractional exponents, but the proof would require a further technical generalization of the Kato-Ponce estimate in Lemma \ref{lem:kato-ponce}). This implies that if we define $\gamma^* = \sup\{ \gamma : V,\bar \rho \in H^\gamma\}$ and if $\max\{d/2,s-1\}\leq \gamma^*<\infty $, then for any $\varepsilon>0$, it holds
\begin{equation}\label{eq:sharp-rate}
    \mathcal{H}(\rho_t|\pi)\leq \Big(\mathcal{H}(\bar \rho|\pi)^{-\frac{s-1}{\gamma^{*}-\eps}} + \frac{t}{C_\varepsilon}\Big)^{-\frac{\gamma^* - \eps}{s-1}}.
\end{equation}

Motivated by the previous remark and numerical experiments (see Figure~\ref{fig:1D-polynomial} where this rate with $\varepsilon=0$ is reported in dotted lines), we conjecture that the $-\gamma^*/(s-1)$ exponent is sharp. 
\end{rmk}

As discussed in the introduction, the energy dissipation~\eqref{eq:stein-dissipation-identity} directly implies global exponential convergence in the case $s=1$.  This result was previously obtained in~\cite{carrillo2024stein} on $\mathbb{R}^d$, where the control on tail behavior makes the analysis more difficult. Here we give a short proof in the compact case, and in addition, we show the global well-posedness of the flow in $L^\infty$ via a maximum principle:

\begin{thm}[Global exponential convergence for $s=1$]\label{thm:quantitative-convergence-s=1} Let $s=1$ and $\gamma>\tfrac{d}{2}$, $\gamma\ge 1$. Let $V\in H^{\gamma}(\T^{d})$ and $\pi \in \mathcal{P}\cap H^{\gamma}(\T^{d})$ be defined as in \eqref{eq:def-pi-intermsof-V}. Then, there exists $\alpha>0$ depending only on $d$ and $\lVert V\rVert_{H^{\gamma}}$ such that for every $\bar\rho \in \mathcal{P}\cap L^{\infty}(\T^{d})$, there is a unique global-in-time, uniformly bounded solution $\rho\in C_{t}\mathcal{P}_{x}\cap  L^{\infty}_{t,x}$ of \eqref{eq:SVGD}, and
\begin{equation}\label{eq:exponential-decay-s=1}
    \mathcal{H}(\rho_{t}|\pi)\le \mathcal{H}(\bar\rho|\pi)e^{-\alpha t}\qquad \forall t\ge 0.
\end{equation}
\end{thm}

In Theorems \ref{thm:quantitative-convergence-s>1} and \ref{thm:quantitative-convergence-s=1}, we focus on the quantitative convergence, while the existence of a local solution in a suitable class is assumed. However,  by the arguments detailed in \cite[Section 2]{chizat2026quantitative}, the local well-posedness (existence and uniqueness/stability) of weak solutions holds in $L^\infty$, for $s=1$, in $L^{p}$ with $p>\tfrac{d}{2s-2}$ for $s\in (1,\tfrac{d}{2}+1]$, and in the space of measures for $s>\tfrac{d}{2}+1$. In fact, for solutions in such regularity classes, the velocity field $v_t$ in \eqref{eq:SVGD} is Lipschitz for $s>1$, and log-Lipschitz for $s=1$. Of course, local solutions in these classes can then be extended uniquely as long as the corresponding norm remains bounded in time. In particular, the global well-posedness in the case $s=1$ follows from the a priori uniform bound we obtain in Lemma \ref{lem:maximum-principle-s=1}. Moreover, the equation has propagation of H\"{o}lder and Sobolev regularity from the data. 

\paragraph{Extension to other kernels} 
Although the previous statement is formulated for Riesz kernels, the proof only uses structural properties of its Fourier multiplier. In particular, the argument is stable under replacing $(-\Delta)^{-s}$ by a symmetric convolution operator $\mathcal K f = K * f$ with the same order, the same $L^p$ multiplier bounds, and the same leading dissipative part up to lower-order errors, as shown in Theorem~\ref{thm:general_kernels}. As an illustration, the conclusions of Theorem~\ref{thm:quantitative-convergence-s>1} also hold for the ``negative distance-like'' kernel
\begin{equation}\label{eq:negative-distance-kernel}
    K(x)=\left(1-4|x|\right)_+^{d+2},
\end{equation}
in the regime $s=(d+1)/2$, see Remark \ref{rmk:Askey-kernel}. This kernel can be conveniently evaluated in any dimension, in contrast to  Riesz kernels, which have no closed form on $\mathbb{T}^d$.

{
\subsection{Link with Wasserstein gradient flows of KMD and ideas of the proof}\label{subsec:idea-proof}

Our analysis is closely related to the recent work \cite{chizat2026quantitative} on Wasserstein gradient flows of Kernel Mean Discrepancy (KMD, or MMD) functionals. In that setting, given a target measure $\nu$ on a compact manifold $M$, one considers
\begin{equation*}
    \mathscr E^\nu(\mu):=\frac12\int_M\int_M K(x,y)d(\mu-\nu)(x)d(\mu-\nu)(y),
\end{equation*}
and its Wasserstein gradient flow
\begin{equation}\label{eq:PDE-general-intro}
    \partial_{t}\mu_{t}=\diver \left(\mu_{t}\nabla \mathcal{K}\left(\mu_{t}-\nu\right)\right)\ \ \text{in $(0,T)\times M$},\quad\text{where}\quad \mathcal{K}(\eta)(x):=\int_{M}K(x,y)d\eta(y).
\end{equation}
For the Riesz kernels on $\T^d$, $\mathcal K=(-\Delta)^{-s}$ and $\mathscr E^\nu(\mu)=\frac12\lVert \mu-\nu\rVert_{\dot H^{-s}}^2$. The equation studied here has the same active-scalar structure (see
 \eqref{eq:rewriting-velocity} below). In fact, when $V$ is constant, our dynamics coincide exactly with the KMD flow \eqref{eq:PDE-general-intro} with target $\nu=1$; otherwise, there are additional drift terms generated by the fixed background density $\pi$.

There is, however, an important difference in the Lyapunov structure. In the KMD flow, the decreasing energy is the weak Sobolev discrepancy $ \mathscr{E}^\nu(\mu) = \frac12 \lVert \mu_t-\nu\rVert_{\dot H^{-s}}^2$. In the SVGD flow, the decreasing functional is the relative entropy, and the dissipation is the kernel Stein discrepancy, which under the fixed lower bound on $\pi$ satisfies (see Lemma \ref{lem:comparability-stein-H1-s})
\begin{equation*}
    I_s(\rho|\pi)=\left\lVert \pi\nabla\frac{\rho-\pi}{\pi}\right\rVert_{\dot H^{-s}}^2 \approx \lVert \rho-\pi\rVert_{\dot H^{1-s}}^2.
\end{equation*}
This coercivity depends only on the reference density $\pi$, and not on a lower bound for the evolving density $\rho_t$. This is a major difference with respect to the KMD dissipation, where coercivity is tied to lower bounds on $\mu_t$\footnote{Unconditional weak convergence (i.e.~the analog of Proposition~\ref{prop:qualitative-convergence-SVGD}) does not hold for the KMD flow, mainly for this reason.}.

The quantitative decay follows from a local \L{}ojasiewicz inequality along the flow. Combining the previous comparison with the entropy bound \eqref{eq:rel-entropy-upper-bound-L2} in terms of $\| \rho-\pi\|_{L^2}$, and using Sobolev interpolation, for $s>1$ we obtain the following: if $\sigma=\rho-\pi\in H^\gamma$ for some $\gamma> \max\{d/2, s-1\}$, then
\begin{equation}\label{eq:lojasiewicz-intro-SVGD}
    I_s(\rho|\pi)\gtrsim \mathcal{H}(\rho|\pi)^{1+\frac{s-1}{\gamma}}
    \left(\lVert \rho-\pi\rVert_{\dot H^\gamma}^2\right)^{-\frac{s-1}{\gamma}}.
\end{equation}
When $s=1$, the high norm $H^{\gamma}$ drops out: $I_1(\rho|\pi)\gtrsim \lVert \rho-\pi\rVert_{L^2}^2\gtrsim \mathcal{H}(\rho|\pi)$, and the entropy dissipation identity immediately gives exponential decay by Gr\"onwall. The maximum principle in Lemma \ref{lem:maximum-principle-s=1} is then used to keep the solution in the global well-posedness class; the convergence mechanism itself does not require imposing a positive lower bound on $\rho_t$.

For $s>1$, the \L{}ojasiewicz inequality \eqref{eq:lojasiewicz-intro-SVGD} is effective only if the high Sobolev norm $H^{\gamma}$ stays controlled. This is obtained using the strategy of \cite{chizat2026quantitative}: the high Sobolev norm is not expected to dissipate monotonically, but its growth can be controlled. Specifically, Lemma \ref{lem:energy-estimate-higher-order} propagates $\lVert \sigma_t\rVert_{\dot H^\gamma}$ up to the integral of a Lipschitz norm of the velocity field, and Lemma \ref{lem:time-integral-Lipschitz-field} bounds this integral by a small power of the initial entropy. By a bootstrap argument, this shows that the solution remains for all times in a region where a \L{}ojasiewicz inequality holds, yielding the polynomial entropy decay of Theorem \ref{thm:quantitative-convergence-s>1}. The Sobolev decay follows by interpolation.

}

\section{Numerical experiments}\label{sec:numerics}

We perform numerical experiments in two settings. On the 1D torus (see Figure~\ref{fig:1D}), we solve the mean-field PDE~\eqref{eq:SVGD} directly. Since the equation can be written in divergence form, we use an upwind finite-volume scheme that preserves mass and positivity, and the velocity field is computed using multipliers in Fourier domain. We use a grid of $2048$ points and adaptive CFL step sizes.

On the 2D torus (see Figure~\ref{fig:2D}), we instead implement the fully discrete SVGD particle iterations. For computational efficiency, we use the  ``negative-distance'' kernel~\eqref{eq:negative-distance-kernel}
rather than a Riesz kernel, whose evaluation is more costly since it does not have a closed-form expression. This kernel has the same spectral decay as a Riesz kernel with parameter \(s=(d+1)/2\) and our results apply; see Appendix~\ref{app:general-kernel}. The \(L^2\)-error is estimated through a truncated Fourier expansion. In this setting, the Fourier truncation error, together with finite-particle fluctuations, makes the convergence rates harder to read off cleanly but we can still observe our predictions at a more qualitative level\footnote{We stop plotting the curves once discretization errors visibly dominate. These errors may, however, already affect the slopes before that point, since they are not uniform across Fourier modes. In particular, the curves artificially accelerate when the leading residual error is carried by high-frequency modes that are suppressed by the truncated Fourier estimator of the $L^2$-norm.}. We use $2000$ particles, a step-size of $0.05$, and $16\times 16$ lowest Fourier modes for the $L^2$-error estimation.

In both settings, we consider uniform initialization $\bar \rho=1$, and we generate a random potential $V$ of the desired Sobolev regularity by sampling its Fourier coefficients with centered Gaussian weights with the appropriate variance decay. This allows us to directly control the largest value of $\gamma^*$ such that $V \in H^{\gamma^*}$ that appears in the convergence rate~\eqref{eq:sharp-rate}.

\begin{figure}[htb]
    \centering

    \begin{subfigure}[c]{0.38\textwidth}
        \centering
        \includegraphics[width=\textwidth, trim=0cm 0cm 0cm 0cm, clip]{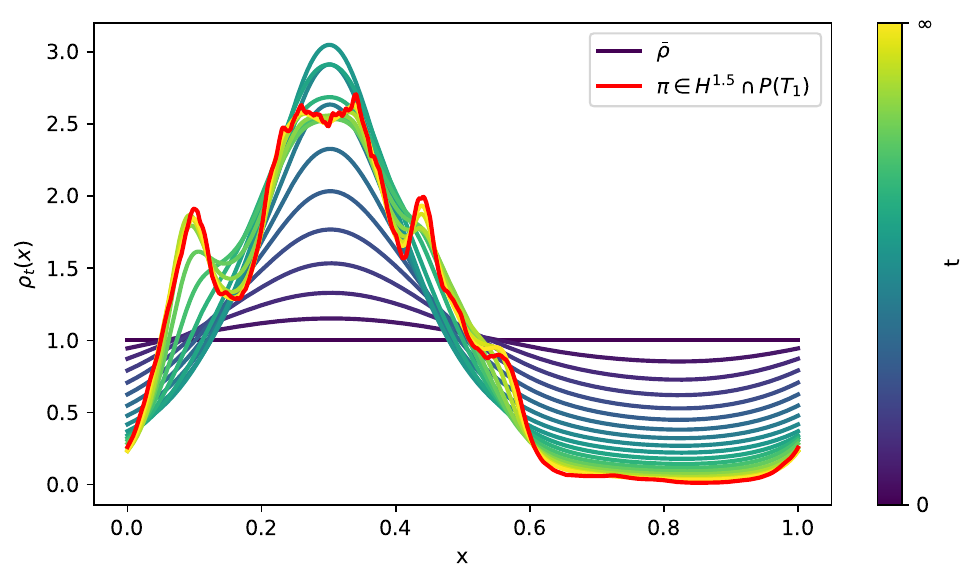}
        \caption{Evolution of the flow ($s=2$)}
        \label{fig:1D-evolution}
    \end{subfigure}
    \hfill
    \begin{subfigure}[c]{0.29\textwidth}
        \centering
        \includegraphics[width=\textwidth]{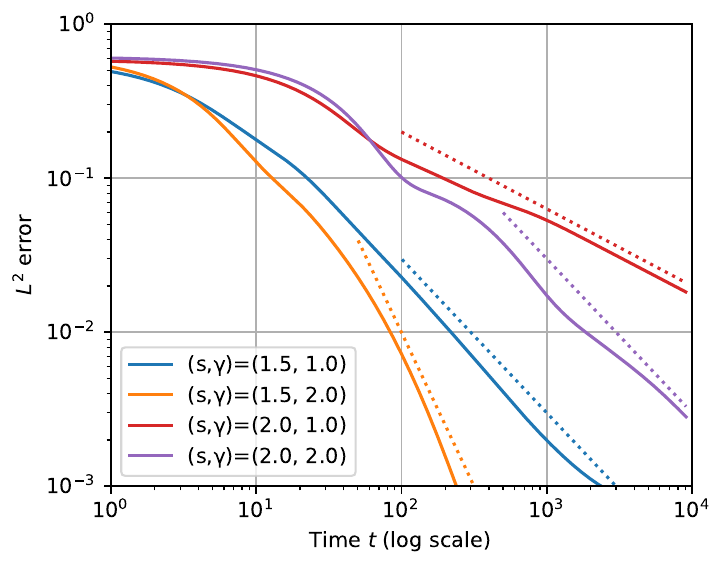}
        \caption{Polynomial rates for $s>1$}
        \label{fig:1D-polynomial}
    \end{subfigure}
    \hfill
    \begin{subfigure}[c]{0.29\textwidth}
        \centering
        \includegraphics[width=\textwidth]{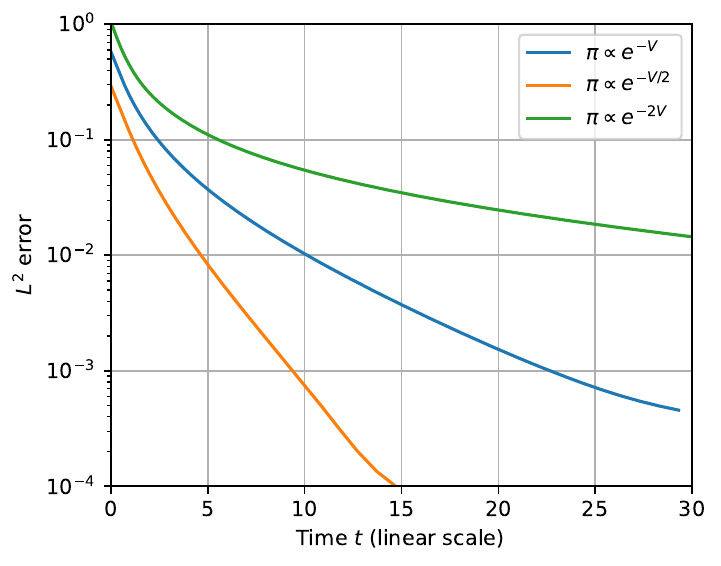}
        \caption{Exponential rates for $s=1$}
        \label{fig:1D-exponential}
    \end{subfigure}

    \caption{Mean-field SVGF in 1D, solved with a finite-volume method (upwind scheme). (a) Evolution of the density for $s=2$ and target in $H^{1.5}$. (b) For $s>1$, the experimental errors (solid lines) and (the square-root of) the theoretical rates of~\eqref{eq:sharp-rate} (dotted lines) are in good agreement, confirming the sharpness of our analysis. (c) For $s=1$, convergence is exponential, with a rate depending on $\Vert V\Vert_{H^{1+\varepsilon}}$ as shown in Theorem~\ref{thm:quantitative-convergence-s=1}. }
    \label{fig:1D}
\end{figure}

\begin{figure}[htb]
    \centering

    \begin{subfigure}[c]{0.30\textwidth}
        \centering
        \raisebox{0.0cm}{%    
        \includegraphics[scale=0.40,trim= 0cm 0cm 0cm 1cm,clip]{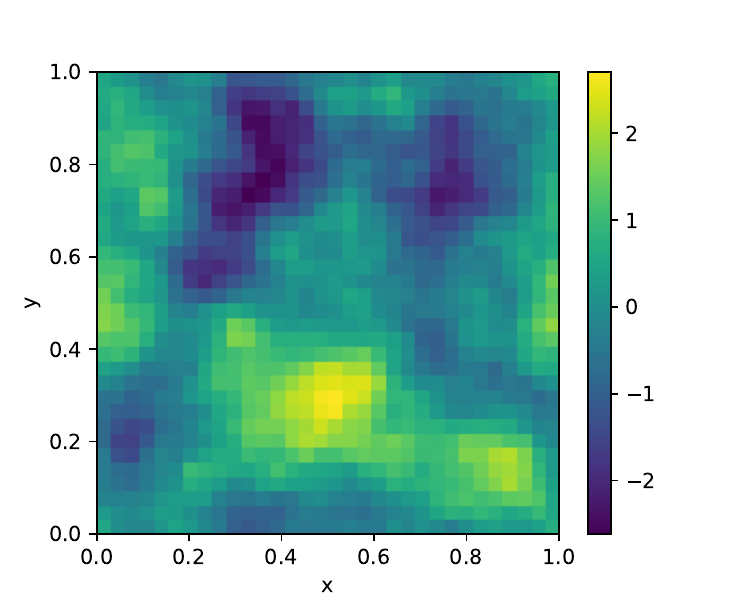}
        }%
        %\vskip -0.1cm
        %\vspace{-0.3cm}%
        \caption{Random potential $V\in H^{1}$}
        \label{fig:2D-V}
    \end{subfigure}
    \hfill
    \begin{subfigure}[c]{0.30\textwidth}
        \centering
        \includegraphics[scale=0.40]{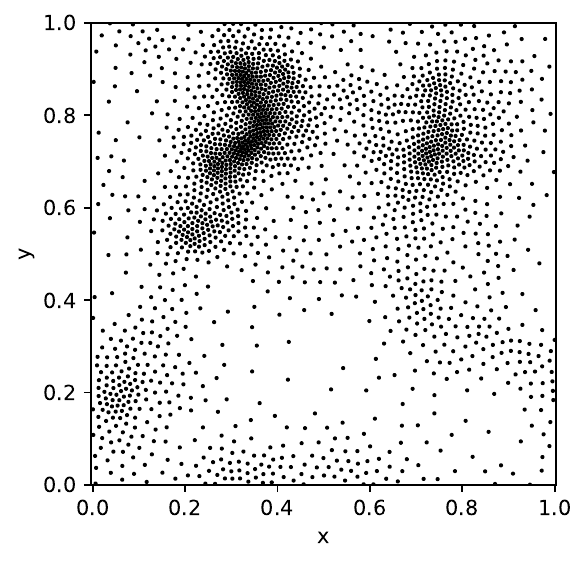}
        \caption{Output of SVGD}
        \label{fig:2D-output}
    \end{subfigure}
    \hfill
    \begin{subfigure}[c]{0.34\textwidth}
        \centering
        \includegraphics[scale=0.39]{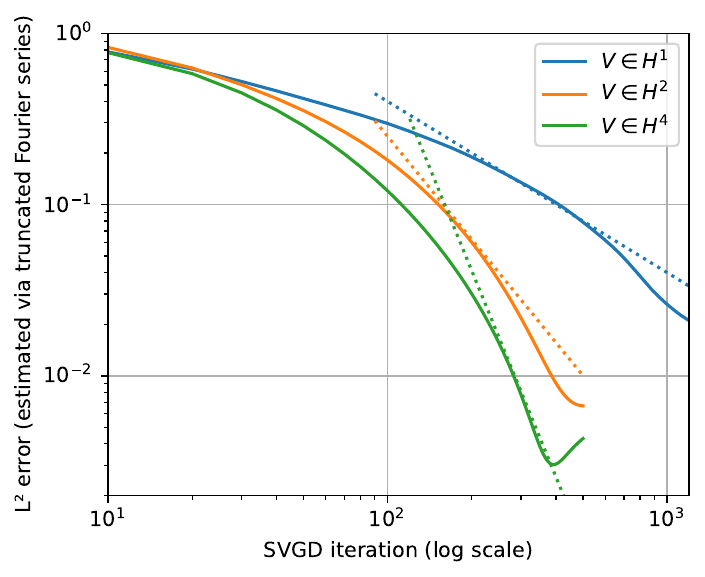}
        \caption{Convergence rates for $s=1.5$}
        \label{fig:2D-error}
    \end{subfigure}

    \caption{SVGD in 2D solved via the interacting particle system, with distance-like kernel ($s=1.5$). (a) A realization of a Gaussian process $V$ with $H^1$ regularity. (b) Output of SVGD for this potential $V$. (c) Convergence of an estimation of $L^2$-error for various target regularities, compared with the (square-root of the) theoretical rates of~\eqref{eq:sharp-rate} (dotted lines). The agreement is weaker than in Figure~\ref{fig:1D}, in particular because of the difficulty to accurately estimate the mean-field $L^2$-error using particles.}
    \label{fig:2D}
\end{figure}

\section{Proof of the main results}\label{sec:proof}

\subsection{Basic estimates on relative entropy and kernel Stein discrepancy}
The Fourier coefficients of a general distribution $f$ on the $d$-torus $\T^{d}$ are given by $\widehat{f}_{k}:=\langle f, e^{2\pi i k\cdot x}\rangle$, where $k\in \Z^{d}$. For any $\beta\in \R$, inhomogeneous and homogeneous $\beta$-Sobolev norms are defined, respectively, as follows:
\begin{equation*}
    \lVert f\rVert_{H^{\beta}}^{2}:= \sum_{k\in \Z^{d}}(1+|2\pi k|^{2})^{\beta}|\widehat{f}_{k}|^{2},\qquad \lVert f\rVert_{\dot H^{\beta}}^{2}:= \sum_{k\in \Z^{d}\setminus \{0\}}|2\pi k|^{2\beta}|\widehat{f}_{k}|^{2}.
\end{equation*}
We denote by $H^{\beta}(\T^{d})$ the space of distributions $f$ for which $\lVert f\rVert_{H^{\beta}}<\infty$. In the sequel we will use the notation $D^{\beta}=(-\Delta)^{\beta/2}$ to denote the multiplier operator defined by the formula 
\begin{equation*}
    D^{\beta}f= \sum_{k\in \Z^{d}\setminus \{0\}}|2\pi k|^{\beta}\widehat{f}_{k}e^{2\pi i k\cdot x}.
\end{equation*}
The operators $D^{\beta}$ are self-adjoint, and commute with standard derivatives. Moreover, homogeneous Sobolev norms can be written as 
\begin{equation*}
    \lVert f\rVert_{\dot H^{\beta}}= \lVert D^{\beta}f\rVert_{L^{2}}.
\end{equation*}
For any $s\ge 1$, the Riesz kernel $K_{s}$ is defined by the expansion
\begin{equation*}
    K_{s}:= \sum_{k\in \Z^{d}\setminus \{0\}}|2\pi k|^{-2s}e^{2\pi i k\cdot x}.
\end{equation*}
In particular, the convolution operator $K_{s}*$ corresponds precisely to $D^{-2s}=(-\Delta)^{-s}$. 

The velocity field generated by a solution $\rho_{t}$ of \eqref{eq:SVGD} can be rewritten in one of the following equivalent forms:
\begin{equation}\label{eq:rewriting-velocity}
    v_{t}= -K_{s}*\left(\rho_{t}\nabla \log\left(\frac{\rho_{t}}{\pi}\right)\right)= -D^{-2s}(\pi \nabla \frac{\rho_{t}-\pi}{\pi})= -\nabla D^{-2s}(\rho_{t}-\pi)-D^{-2s}((\rho_{t}-\pi)\nabla V),
\end{equation}
where the last identity was derived using that $\nabla$ commutes with $D^{-2s}$, and $-\frac{\nabla \pi}{\pi}=\nabla V$.

The first ingredient is the elementary comparability between relative entropy and $L^{2}$-norm of the difference between two densities enjoying suitable uniform bounds:
\begin{lem}\label{lem:comparability-entropy-L2}
    Let $\pi\in \mathcal{P}(\T^{d})$ be such that $0< \lambda\le \pi\ll \mathcal{L}^{d}$. Then, the following bound holds for all $\rho\in \mathcal{P}(\T^{d})$:
    \begin{equation}\label{eq:rel-entropy-upper-bound-L2}
        \mathcal{H}(\rho|\pi)\le\frac{1}{\lambda}\lVert \rho-\pi\rVert_{L^{2}}^{2}.
    \end{equation}
    If in addition $\rho, \pi \le \Lambda$, then
    \begin{equation}\label{eq:rel-entropy-lower-bound-L2}
        \mathcal{H}(\rho|\pi)\ge \frac{1}{2\Lambda}\lVert \rho-\pi\rVert_{L^{2}}^{2}.
    \end{equation}
\end{lem}
\begin{proof}
Since $\rho, \pi$ are both probability measures, we may write
\begin{equation*}
    \mathcal{H}(\rho|\pi)= \int_{\T^{d}} \left(\log\left(\frac{\rho}{\pi}\right)\frac{\rho}{\pi}-\frac{\rho}{\pi}+1\right)\pi.
\end{equation*}
The entropy function $\varphi(r):= \log(r)r-r+1$ satisfies $\varphi(r)\le (r-1)^{2}$. Therefore
\begin{equation*}
    \mathcal{H}(\rho|\pi)\le \int_{\T^{d}} \left(\frac{\rho}{\pi}-1\right)^{2}\pi\le \frac{1}{\lambda}\lVert \rho-\pi\rVert_{L^{2}}^{2}.
\end{equation*}
Moreover, it holds $\varphi(r)\ge \tfrac{1}{2\max\{r,1\}}(r-1)^{2}$. Therefore, if $\rho, \pi\le \Lambda$,
\begin{equation*}
    \mathcal{H}(\rho|\pi)\ge \int_{\T^{d}} \frac{1}{2\max\{\tfrac{\rho}{\pi},1\}}\left(\frac{\rho}{\pi}-1\right)^{2}\pi=\int_{\T^{d}} \frac{(\rho-\pi)^{2}}{2\max\{\rho,\pi\}}\ge \frac{1}{2\Lambda}\lVert \rho-\pi\rVert_{L^{2}}^{2}.\qedhere
\end{equation*}
\end{proof}

We next identify the kernel Stein discrepancy with the homogeneous Sobolev norm naturally associated with the Riesz kernel.

\begin{lem}\label{lem:comparability-stein-H1-s}
Let $s\ge 1$, and let $\pi\in \mathcal{P}\cap H^{\alpha}(\T^{d})$ with $\alpha>\tfrac{d}{2}$ and $\alpha\ge s$. Suppose that $\pi\ge \lambda>0$. Then, the following holds for every $\rho\in \mathcal{P}(\T^{d})$:
    \begin{equation}\label{eq:approx-formula-SFI}
        I_{s}(\rho|\pi)\approx_{d,s,\alpha, \lambda, \lVert \pi\rVert_{H^{\alpha}}} \lVert \rho-\pi\rVert_{\dot H^{1-s}}^{2}.
    \end{equation}
    In particular, if $s>1$ and $\rho-\pi\in H^{\gamma}(\T^{d})$ for some $\gamma>0$, then 
    \begin{equation}\label{eq:coercivity-interpolation-SFI}
        I_{s}(\rho|\pi)\gtrsim_{d,s,\alpha,\gamma,\lambda, \lVert\pi\rVert_{H^{\alpha}}} \mathcal{H}(\rho|\pi)^{1+\frac{s-1}{\gamma}}\left(\lVert \rho-\pi\rVert_{\dot H^{\gamma}}^{2}\right)^{-\frac{s-1}{\gamma}}.
    \end{equation}
\end{lem}
\begin{proof}
    We consider the zero-mean measure $\sigma:=\rho-\pi$. Note that, by the definition of KSD, 
    applying Lemma \ref{lem:equiv-homogeneous-norms-smooth-multiplier} twice with smooth multipliers $\pi$, $\pi^{-1}$, and zero-mean functions $\nabla\frac{\sigma}{\pi}$, $\sigma$, respectively, we deduce
    \begin{equation*}
         I_{s}(\rho|\pi)=\lVert \pi \nabla \frac{\sigma}{\pi}\rVert_{\dot H^{-s}}^{2}\approx_{d,s,\alpha,\lVert \pi\rVert_{H^{\alpha}}} \lVert \nabla \frac{\sigma}{\pi}\rVert_{\dot H^{-s}}^{2}=\lVert \frac{\sigma}{\pi}\rVert_{\dot H^{1-s}}^{2}\approx_{d,s,\alpha,\lambda,\lVert \pi\rVert_{H^{\alpha}}} \lVert \sigma\rVert_{\dot H^{1-s}}^{2},
    \end{equation*}
    which gives \eqref{eq:approx-formula-SFI}. 
    Assuming $\rho-\pi\in H^{\gamma}(\T^{d})$ for some $\gamma>0$, \eqref{eq:rel-entropy-upper-bound-L2} and Sobolev interpolation of homogeneous norms give
    \begin{equation*}
        \mathcal{H}(\rho|\pi)
        \lesssim_{\lambda}
        \lVert \sigma\rVert_{L^{2}}^2\le \lVert \sigma\rVert_{\dot H^{1-s}}^{\frac{2\gamma}{\gamma+s-1}}\lVert \sigma\rVert_{\dot H^{\gamma}}^{\frac{2(s-1)}{\gamma+s-1}}
        \lesssim_{d,s, \alpha, \gamma, \lambda, \lVert \pi\rVert_{H^{\alpha}}}
         I_{s}(\rho|\pi) ^{\frac{\gamma}{\gamma+s-1}}\lVert \sigma\rVert_{\dot H^{\gamma}}^{\frac{2(s-1)}{\gamma+s-1}},
    \end{equation*}
    which rewrites as \eqref{eq:coercivity-interpolation-SFI}.
\end{proof}

\subsection{Controls on propagation of regularity}

Let us now state the key technical estimates which ensure that the flow does not lose regularity too fast. Their proofs can be found in Appendix~\ref{app:propagation-of-regularity}. The main a priori estimate is the following high-order energy inequality, which isolates the dissipative linear part from the transport terms.
\begin{lem}\label{lem:energy-estimate-higher-order}
    Let $s\ge 1$ and $\gamma>\tfrac{d}{2}$. Let $V\in H^{\gamma+s}(\T^{d})$ and $\pi\in \mathcal{P}\cap H^{\gamma+s}(\T^{d})$ be defined as in \eqref{eq:def-pi-intermsof-V}. Let $\rho\in C_{t}\mathcal{P}_{x}\cap L_{t}^{\infty}H^{\gamma}_{x}$ be a solution of \eqref{eq:SVGD} defined up to some time $T>0$, and call $\sigma_{t}=\rho_{t}-\pi$. Then, there is a constant $C>0$ depending only on $d,s,\gamma$, and $\lVert V\rVert_{H^{\gamma+s}}$ such that the following estimate holds:
    \begin{equation}\label{eq:energy-estimate-higher-order}
        \frac{d}{dt}\lVert \sigma_{t}\rVert_{\dot H^{\gamma}}^{2}\le CL(t)\lVert \sigma_{t}\rVert_{\dot H^{\gamma}}^{2}-2\big(\min_{\T^{d}}\pi\big)\lVert \sigma_{t}\rVert_{\dot H^{\gamma-s+1}}^{2}+C\lVert \sigma_{t}\rVert_{\dot H^{\gamma-s}}\lVert \sigma_{t}\rVert_{\dot H^{\gamma-s+1}}\qquad \forall t\in (0,T),
    \end{equation}
    where we set
    \begin{equation}\label{eq:def-L-lipschitz-field}
        L(t):=\lVert \nabla^{2}D^{-2s}\sigma_{t}\rVert_{L^{\infty}}+\lVert \nabla D^{-2s}(\sigma_{t}\nabla V)\rVert_{L^{\infty}}.
    \end{equation}
\end{lem}

It remains to control the cumulative effect of the Lipschitz norm $L(t)$.

\begin{lem}\label{lem:time-integral-Lipschitz-field}
    Under the same assumptions and using the same notation as in Lemma \ref{lem:energy-estimate-higher-order}, assume further that $\gamma>s-1$ and $\lVert \sigma_{t}\rVert_{\dot H^{\gamma}}^{2}\le M$ for all $t\in [0,T)$. Then, there is $\beta=\beta(d,s,\gamma)>0$ such that 
    \begin{equation*}
        \int_{0}^{T}L(t)\,dt\lesssim_{d,s,\gamma,\lVert V\rVert_{H^{\gamma+s}}, M} \mathcal{H}(\bar\rho|\pi)^{\beta}.
    \end{equation*}
\end{lem}

\subsection{Proof of the main results}

We are now ready to prove our main result, Theorem \ref{thm:quantitative-convergence-s>1}:
\begin{proof}[Proof of Theorem \ref{thm:quantitative-convergence-s>1}] Let $T>0$ be the maximal time of existence in $C_{t}\mathcal{P}_{x}\cap L^{\infty}_{t}H^{\gamma}_{x}$ of the solution $\rho_{t}$, and let $\sigma_{t}:=\rho_{t}-\pi$. We call $M:=2\lVert \bar\rho-\pi\rVert_{\dot H^{\gamma}}^{2}$ and consider the first time $\tau>0$ such that $\lVert \sigma_{t}\rVert_{\dot H^{\gamma}}^{2}$ exceeds $M$, i.e. 
\begin{equation*}
    \tau:=\sup\left\{t\in (0,T]:\lVert \sigma_{r}\rVert_{\dot H^{\gamma}}^{2}< M,\,\, \forall r\in [0,t)\right\}>0.
\end{equation*}
Combining \eqref{eq:stein-dissipation-identity} with \eqref{eq:coercivity-interpolation-SFI}, from the definition of $\tau$ we deduce that 
\begin{equation*}
    \frac{d}{dt}\mathcal{H}(\rho_{t}|\pi)\lesssim_{d,s,\gamma,\lVert V\rVert_{H^{\gamma+s}},M} -\mathcal{H}(\rho_{t}|\pi)^{1+\frac{s-1}{\gamma}}\qquad \forall t\in (0,\tau),
\end{equation*}
which integrated yields
\begin{equation*}
    \mathcal{H}(\rho_{t}|\pi)\le \left(\mathcal{H}(\bar\rho|\pi)^{-\frac{s-1}{\gamma}}+\frac{t}{C}\right)^{-\frac{\gamma}{s-1}}\qquad \forall t\in [0,\tau),
\end{equation*}
where $C=C(d,s,\gamma,\lVert V\rVert_{H^{\gamma+s}}, M)>0$. The thesis of the Theorem would follow if we could prove that $\tau=T$. In fact, in this case, by the definition of $\tau$, the solution would be uniformly bounded in $H^{\gamma}(\T^{d})$ up to the maximal time of existence, thus forcing $\tau=T=\infty$.

First observe that for any given $\eps>0$, Sobolev interpolation of homogeneous norms, together with Young's inequality, yields
\begin{equation*}
    \lVert \sigma_{t}\rVert_{\dot H^{\gamma-s}}\lVert \sigma_{t}\rVert_{\dot H^{\gamma-s+1}}\le \eps \lVert \sigma_{t}\rVert_{\dot H^{\gamma-s+1}}^{2}+C(\gamma,s,\eps)\lVert \sigma_{t}\rVert_{\dot H^{1-s}}^{2}.
\end{equation*}
Choosing $\eps$ sufficiently small depending only on $d,s,\gamma$ and $\lVert V\rVert_{H^{\gamma+s}}$ we may then absorb the third positive term in the right-hand side of \eqref{eq:energy-estimate-higher-order} in the second negative one, up to adding a multiple of $\lVert \sigma_{t}\rVert_{\dot H^{1-s}}^{2}$. Using also \eqref{eq:approx-formula-SFI}, we may then rewrite \eqref{eq:energy-estimate-higher-order} as follows:
\begin{equation}\label{eq:higher-order-energy-estimate-simplified}
    \frac{d}{dt}\lVert \sigma_{t}\rVert_{\dot H^{\gamma}}^{2}\lesssim_{d,s,\gamma,\lVert V\rVert_{H^{\gamma+s}}} L(t)\lVert \sigma_{t}\rVert_{\dot H^{\gamma}}^{2}+I_{s}(\rho_{t}|\pi)\qquad \forall t\in (0,\tau),
\end{equation}
where $L(t)$ is as in \eqref{eq:def-L-lipschitz-field}. In addition, by Lemma \ref{lem:time-integral-Lipschitz-field}, for some $\beta=\beta(d,s,\gamma)>0$ we have
\begin{equation}\label{eq:control-integral-lip-vector-field}
    \int_{0}^{\tau}L(t)\,dt\lesssim_{d,s,\gamma,\lVert V\rVert_{H^{\gamma+s}},M} \mathcal{H}(\bar\rho|\pi)^{\beta}.
\end{equation}
Gr\"onwall's inequality applied to \eqref{eq:higher-order-energy-estimate-simplified} then yields, for all $t\in [0,\tau)$, 
\begin{align*}
    \lVert \sigma_{t}\rVert_{\dot H^{\gamma}}^{2}&\le  \lVert \bar\rho-\pi\rVert_{\dot H^{\gamma}}^{2}\exp\left(C\int_{0}^{t}L(r)\,dr\right)+C\int_{0}^{t}I_{s}(\rho_{r}|\pi)\exp\left(C\int_{r}^{t}L(u)\,du\right)dr\\
    &\le \left(\lVert \bar\rho-\pi\rVert_{\dot H^{\gamma}}^{2}+C\int_{0}^{t}I_{s}(\rho_{r}|\pi)\,dr\right)\exp\left(C\int_{0}^{t}L(r)\,dr\right)\\
    &\le \left(\frac{M}{2}+C\mathcal{H}(\bar\rho|\pi)\right)\exp\left(C\mathcal{H}(\bar\rho|\pi)^{\beta}\right),
\end{align*}
where in the third step we used \eqref{eq:control-integral-lip-vector-field}, the dissipation identity \eqref{eq:stein-dissipation-identity}, and the definition of $M$. If $\mathcal{H}(\bar\rho|\pi)$ is chosen sufficiently small, then the right-hand side of the expression above is strictly less than $M$. By the definition of $\tau$, this implies that $\tau=T$, concluding the proof of global existence and \eqref{eq:polynomial-decay-s>1}. Finally, observe that by Lemma \ref{lem:comparability-entropy-L2}, $\mathcal{H}(\rho_{t}|\pi)\approx \lVert \sigma_{t}\rVert_{L^{2}}^{2}$, therefore \eqref{eq:interpolation-decay-s>1} follows directly from Sobolev interpolation.
\end{proof}

In the edge case $s=1$, once global well-posedness is established, the exponential convergence follows directly as follows.

\begin{proof}[Proof of Theorem \ref{thm:quantitative-convergence-s=1}]
    The solution is global thanks to Lemma \ref{lem:maximum-principle-s=1}. Then, by \eqref{eq:stein-dissipation-identity}, \eqref{eq:approx-formula-SFI}, and \eqref{eq:rel-entropy-upper-bound-L2}, we find
    \begin{equation*}
        \frac{d}{dt}\mathcal{H}(\rho_{t}|\pi)=-I_{s}(\rho_{t}|\pi)\lesssim_{d,\gamma, \lVert V\rVert_{H^{\gamma}}} -\lVert \rho_{t}-\pi\rVert_{L^{2}}^{2}\le - \big(\min_{\T^{d}}\pi \big)\mathcal{H}(\rho_{t}|\pi)\qquad \forall t>0.
    \end{equation*}
    Therefore, \eqref{eq:exponential-decay-s=1} follows from Gr\"onwall's inequality. 
\end{proof}

\section{Conclusion and limitations}\label{sec:conclusion}

We have provided a sharp local convergence analysis of SVGD in the continuous-time mean-field limit. Our results clarify the role of the kernel through its spectral decay and highlight how the regularity of the target affects the convergence rates.

Several limitations of our analysis call for further investigation:
\begin{itemize}
    \item \textbf{From local to global.}
    Except in the case $s=1$, our analysis is local: it requires the source and target to be sufficiently close in $L^2$. In general, it remains unclear whether global $L^2$ convergence holds. In fact, even global well-posedness remains open in the range $s\in (1,\tfrac{d}{2}+1)$. A possible route toward global results is to inject noise into the dynamics, as in~\cite{priser2024long,gallego2018stochastic}. Even an arbitrarily small amount of noise may yield uniform regularity estimates, allowing the solution to eventually enter the local regime studied here.

    \item \textbf{From compact to non-compact domains.}
    We have worked on the torus in order to isolate the spectral mechanism behind the convergence rates, a viewpoint that has been relatively underexploited in the SVGD literature. Extending the analysis to $\mathbb R^d$ is a natural next step.

    \item \textbf{From continuum to discrete dynamics.}
    As discussed in the introduction, existing analyses of practical SVGD algorithms typically rely on analogues of the weak energy estimate~\eqref{eq:basic-rate}. It would be interesting to adapt the sharper approach developed here to discrete particle settings.
\end{itemize}

\appendix

\section*{Appendix}
\addcontentsline{toc}{section}{Appendix}

\paragraph{Contents of the appendix:}
\begin{itemize}
    \item Appendix~\ref{app:estimates}: \nameref{app:estimates}
    \item Appendix~\ref{app:qualitative-convergence}: \nameref{app:qualitative-convergence}
        \item Appendix~\ref{app:propagation-of-regularity}: \nameref{app:propagation-of-regularity}
        \item Appendix~\ref{app:maximum-principle}: \nameref{app:maximum-principle}
     \item Appendix~\ref{app:general-kernel}: \nameref{app:general-kernel}
\end{itemize}

\section{Auxiliary estimates}\label{app:estimates}

We shall repeatedly use the following multiplier estimate to insert smooth fixed weights in homogeneous Sobolev norms.

\begin{lem}\label{lem:equiv-homogeneous-norms-smooth-multiplier}
Let $\alpha, \beta\in \R$ be such that $\alpha>\tfrac{d}{2}$ and $\alpha \ge |\beta|$. Let $f\in H^{\beta}(\T^{d})$ be such that $\widehat{f}_{0}=0$, and let $g\in H^{\alpha}(\T^{d})$. Then, the following bound holds:
\begin{equation*}
    \lVert gf\rVert_{\dot H^{\beta}}\lesssim_{d,\alpha,\beta, \lVert g\rVert_{H^{\alpha}}}\lVert f\rVert_{\dot H^{\beta}}.
\end{equation*}
Assume in addition that $g\ge \lambda>0$. Then 
\begin{equation*}
    \lVert f\rVert_{\dot H^{\beta}}\lesssim_{d,\alpha,\beta, \lambda, \lVert g\rVert_{H^{\alpha}}}\lVert gf\rVert_{\dot H^{\beta}}.
\end{equation*}
\end{lem}
\begin{proof}
    The first claim follows from the following chain of inequalities:
    \begin{equation*}
        \lVert gf\rVert_{\dot H^{\beta}}\lesssim_{d,\beta} \lVert gf\rVert_{H^{\beta}}\lesssim_{d,\alpha, \beta,\lVert g\rVert_{H^{\alpha}}} \lVert f\rVert_{H^{\beta}}\lesssim_{d,\beta} \lVert f\rVert_{\dot H^{\beta}},
    \end{equation*}
    where in the penultimate step we used the fact that multiplication with $g$ is a bounded operator\footnote{Indeed, $\lVert fg\rVert_{L^{2}}\le \lVert g\rVert_{L^{\infty}}\lVert f\rVert_{L^{2}}\lesssim_{d,\alpha} \lVert g\rVert_{H^{\alpha}}\lVert f\rVert_{L^{2}}$, by the Sobolev embedding, since $\alpha>\tfrac{d}{2}$. Moreover, $\lVert fg\rVert_{H^{\alpha}}\lesssim_{d,\alpha} \lVert g\rVert_{H^{\alpha}}\lVert f\rVert_{H^{\alpha}}$, again because of $\alpha>\tfrac{d}{2}$. By interpolation, the multiplication with $g$ is then bounded in $H^{\beta}(\T^{d})$, for all $\beta\in [0,\alpha]$, and by duality, for all $\beta\in [-\alpha,0]$ as well.} in $H^{\beta}(\T^{d})$, while in the last step we used that $f$ has zero average. 

    Now let us assume that $g\ge \lambda>0$, so that $\tfrac{1}{g}\in H^{\alpha}(\T^{d})$, and $\lVert \tfrac{1}{g}\rVert_{H^{\alpha}}\leq C({d,\alpha,\lambda, \lVert g\rVert_{H^{\alpha}}})$. We have 
    \begin{equation*}
        \lVert f\rVert_{\dot H^{\beta}}=\lVert g^{-1}gf\rVert_{\dot H^{\beta}}\lesssim_{d,\beta} \lVert g^{-1}gf\rVert_{H^{\beta}}\lesssim_{d,\alpha,\beta,\lambda, \lVert g\rVert_{H^{\alpha}}} \lVert gf\rVert_{H^{\beta}}\lesssim |\widehat{(gf)}_{0}|+\lVert gf\rVert_{\dot H^{\beta}},
    \end{equation*}
    where we used that multiplication with $\tfrac{1}{g}$ is a bounded operator in $H^{\beta}(\T^{d})$. We only need to bound the average of $gf$ with $\lVert gf\rVert_{\dot H^{\beta}}$. To do that, observe that since $\widehat{f}_{0}=0$, then
    \begin{equation*}
        0=\widehat{(g^{-1}gf)}_{0} =\langle gf-\widehat{(gf)}_{0},g^{-1}\rangle +\widehat{(g^{-1})}_{0} \widehat{(gf)}_{0}. 
    \end{equation*}
    As a consequence
    \begin{equation*}
        |\widehat{(gf)}_{0}|\le \frac{1}{\widehat{(g^{-1})}_{0}}|\langle gf-\widehat{(gf)}_{0},g^{-1}\rangle|\le \frac{1}{\widehat{(g^{-1})}_{0}}\lVert g^{-1}\rVert_{\dot H^{-\beta}}\lVert gf\rVert_{\dot H^{\beta}}\lesssim_{d,\alpha,\beta,\lambda,\lVert g\rVert_{H^{\alpha}}} \lVert gf\rVert_{\dot H^{\beta}},
    \end{equation*}
    as we wanted.
\end{proof}

The following Lemma, which contains Kato-Ponce type commutator estimates in the $d$-dimensional torus, was deduced in \cite{chizat2026quantitative} from the corresponding result on the Euclidean space (see \cite{li2019kato}).
Recall the notation $D^{\beta}:=(-\Delta)^{\beta/2}$, for all $\beta\in \R$.
\begin{lem}[\protect{\cite[Lemma A.6, Remark A.7]{chizat2026quantitative}}]\label{lem:kato-ponce}
    Let $\gamma>0$ and $2\le p_{1},q_{1},p_{2},q_{2}\le \infty$ be such that $1/2=1/p_{i}+1/q_{i}$, $i=1,2$. Then, for every $f,g \in C^{\infty}(\T^{d})$ the following holds:
    \begin{itemize}
        \item [i)] If $\gamma>1$:
    \begin{equation}\label{eq:kato-ponce}
        \lVert D^{\gamma}(fg)-fD^{\gamma}g\rVert_{L^{2}}\lesssim_{d,\gamma,p_{1},q_{1},p_{2},q_{2}} \lVert \nabla f\rVert_{L^{p_{1}}}\lVert D^{\gamma-1}g\rVert_{L^{q_{1}}}+\lVert D^{\gamma}f\rVert_{L^{p_{2}}}\lVert g\rVert_{L^{q_{2}}}.
    \end{equation}
    \item [ii)] If $\gamma=1$:
    \begin{equation}\label{eq:kato-ponce-gamma-1}
        \lVert D^{\gamma}(fg)-fD^{\gamma}g\rVert_{L^{2}}\lesssim_{d,p_{1},q_{1}} \lVert \nabla f\rVert_{L^{p_{1}}}\lVert g\rVert_{L^{q_{1}}}.
    \end{equation}
    \item [iii)] If $\gamma\in (0,1)$:
    \begin{equation}\label{eq:kato-ponce-gamma-in-(0,1)}
        \lVert D^{\gamma}(fg)-fD^{\gamma}g\rVert_{L^{2}}\lesssim_{d,\gamma,p_{1},q_{1}} \lVert D^{\gamma}f\rVert_{L^{p_{1}}}\lVert g\rVert_{L^{q_{1}}}.
    \end{equation}
    \end{itemize}
\end{lem}

\section{Qualitative convergence for Stein Variational Gradient Descent}\label{app:qualitative-convergence}

\subsection{An abstract LaSalle principle}

The following proof is inspired by the LaSalle invariance principle for gradient flows in the Wasserstein space shown in~\cite{carrillo2023invariance}. 

Let $M$ be a compact metric space. Let $\delta:\mathcal P(M)\times\mathcal P(M)\to[0,\infty]$, $\mathcal{E}:\mathcal{P}(M)\to [0,\infty]$, and $\mathcal D:\mathcal P(M)\to[0,\infty]$ be given functionals representing a pseudo-distance, an energy, and a dissipation functional, respectively.  

For a given weakly continuous curve  $[0,\infty)\ni t\mapsto \rho_{t} \in \mathcal{P}(M)$, consider the following assumptions:

\begin{itemize}
\item[(1)] \emph{Energy dissipation.}  The curve satisfies $\mathcal{E}(\rho_{0})<\infty$, and 
\begin{equation}
  \mathcal E(\rho_T)+\int_0^T \mathcal D(\rho_t)\,dt
  \le \mathcal E(\rho_0)\qquad \forall T\ge 0.
  \label{eq:abstract-edi}
\end{equation}

\smallskip
\item[(2)] \emph{Motion estimate.} The following bound holds along the curve: 
\begin{equation}
  \delta(\rho_a,\rho_b)
  \le \int_a^b \mathcal D(\rho_t)^{1/2}\,dt\qquad \forall a, b,\quad 0\le a\le b<\infty.
  \label{eq:abstract-motion-estimate}
\end{equation}

\smallskip
\item[(3)] \emph{Coercivity of the pseudo-distance.}  Whenever $\mu_{n},\nu_{n}, \nu \in \mathcal{P}(M)$ are such that $\nu_{n}\rightharpoonup \nu$ and $\delta(\mu_{n},\nu_n)\to 0$, then $\mu_{n}\rightharpoonup \nu$. 

\smallskip
\item[(4)] \emph{Lower semicontinuity of the dissipation functional.}  The map $\mathcal{D}$ is
lower semicontinuous with respect to the weak topology of $\mathcal{P}(M)$.

\smallskip
\item[(5)] \emph{Unique equilibrium state.}  There is $\pi \in \mathcal{P}(M)$ such that 
\begin{equation*}
  \mathcal D(\mu)=0
  \qquad\Longleftrightarrow\qquad
  \mu=\pi.
\end{equation*}

\end{itemize}

\smallskip
\begin{prop}[Abstract LaSalle principle]
\label{prop:abstract-lasalle}
Assume \emph{(1)--(5)}.  Then
\[
  \rho_t\rightharpoonup\pi
  \qquad\text{as}\quad t\to\infty.
\]
\end{prop}

\begin{proof}
Since $\mathcal{E}(\rho_{0})<\infty$ and $\mathcal E\ge0$, the energy-dissipation estimate
\eqref{eq:abstract-edi} implies
\begin{equation}
  \int_0^{\infty}\mathcal D(\rho_t)\,dt<\infty.
  \label{eq:finite-total-dissipation}
\end{equation}
Let $t_n\to\infty$.  Since $M$ is compact, $\mathcal P(M)$ is weakly
compact. Therefore, after extracting a subsequence, not relabeled, there exists
$\rho_\infty\in\mathcal P(M)$ such that $\rho_{t_n}\rightharpoonup \rho_{\infty}$. We wish to prove that $\rho_\infty=\pi$. From \eqref{eq:finite-total-dissipation} and $\mathcal{D}\ge 0$, we deduce that
\begin{equation*}
  \int_{t_n}^{t_n+1}\mathcal D(\rho_t)\,dt\to0.
\end{equation*}
For every fixed \(r\in[0,1]\), the motion estimate \eqref{eq:abstract-motion-estimate} and Cauchy-Schwarz
give
\begin{equation*}
  \delta(\rho_{t_n+r},\rho_{t_n})
  \le \int_{t_n}^{t_n+r}\mathcal D(\rho_t)^{1/2}\,dt  \le 
      \left(\int_{t_n}^{t_n+1}\mathcal D(\rho_t)\,dt\right)^{1/2}
   \to0.
\end{equation*}
By the coercivity assumption $(3)$, we thus obtain
\begin{equation*}
  \rho_{t_n+r}\rightharpoonup\rho_\infty
  \qquad \forall r\in[0,1].
\end{equation*}
The lower semicontinuity of $\mathcal D$, together with Fatou's lemma, yields
\begin{equation*}
  \mathcal D(\rho_\infty)
  \le \int_0^1
       \liminf_{n\to\infty}\mathcal D(\rho_{t_n+r})\,dr \le \liminf_{n\to\infty}
       \int_{t_n}^{t_n+1}\mathcal D(\rho_t)\,dt
   =0.
\end{equation*}
Thus $\mathcal D(\rho_\infty)=0$. Finally, by $(5)$, we get
$\rho_\infty=\pi$, concluding the proof.   
\end{proof}

\subsection{Proof of qualitative convergence}

 The abstract LaSalle principle above applies to a broad class of kernels, provided the associated dissipation and motion estimates satisfy assumptions (1)--(5). We now verify these assumptions for the Riesz kernel $K_s$ used in \eqref{eq:SVGD}, with the choices $\mathcal{E}:= \mathcal{H}(\cdot|\pi)$, $\mathcal{D}= I_{s}(\cdot|\pi)$ and $\delta=W_{1}$, the $1$-Wasserstein distance, which metrizes weak convergence in $\mathcal{P}(\T^{d})$. 

\begin{lem}\label{lem:motion-estimate-SVGD}
    Let $s>\tfrac{d}{2}$, $V\in C^{1}(\T^{d})$, and let $\rho \in C_{t}\mathcal{P}_{x}$ be a solution of \eqref{eq:SVGD} defined up to some positive time $T$. Then, the following estimate holds:
    \begin{equation}\label{eq:bound-W1-integral-SFI}
        W_{1}(\rho_{a},\rho_{b})\lesssim_{d,s} \int_{a}^{b}I_{s}(\rho_{t}|\pi)^{1/2}\,dt\qquad \forall a,b,\quad 0\le a\le b<T.
    \end{equation}
\end{lem}
\begin{proof}
    Let $f:\T^{d}\to \R$ be a smooth function with $\lVert \nabla f\rVert_{L^{\infty}}\le 1$. Since $s>\tfrac{d}{2}$, we can embed $H^{s}(\T^{d})$ into $C^{0}(\T^{d})$. Therefore, the following bound holds:
    \begin{equation*}
        \lVert \rho_{t}\nabla f\rVert_{\dot H^{-s}}=\sup_{\hat {g}_{0}=0,\lVert g\rVert_{\dot H^{s}}\le 1}\langle \rho_{t}\nabla f,g\rangle \le\sup_{\hat {g}_{0}=0,\lVert g\rVert_{\dot H^{s}}\le 1}\lVert g\rVert_{C^{0}}\lesssim_{d,s} 1.
    \end{equation*}
    Using equation \eqref{eq:SVGD}, for almost every $t$, we may bound
\begin{align*}
    \frac{d}{dt}\int_{\T^{d}}f\,d\rho_{t}&=-\int_{\T^{d}}K_{s}*(\pi\nabla \frac{\rho_{t}}{\pi})\cdot \nabla f \,d\rho_{t}\le \lVert \pi\nabla \frac{\rho_{t}}{\pi}\rVert_{\dot H^{-s}}\lVert \rho_{t}\nabla f\rVert_{\dot H^{-s}}\lesssim_{d,s}I_{s}(\rho_{t}|\pi)^{1/2}.
\end{align*}
Integrating this inequality from $a$ to $b$, we obtain
\begin{equation*}
    \left|\int_{\T^{d}}f\,d(\rho_{b}-\rho_{a})\right|\lesssim_{d,s} \int_{a}^{b}I_{s}(\rho_{t}|\pi)^{1/2}\,dt.
\end{equation*}
Then, \eqref{eq:bound-W1-integral-SFI} follows from the arbitrariness of $f$ and the Kantorovich duality formula for the $1$-Wasserstein distance\footnote{Namely, $W_{1}(\mu,\nu)=\sup\left\{\int_{\T^{d}}f\, d(\mu-\nu): \lVert \nabla f\rVert_{L^{\infty}}\le 1\right\}$ for all $\mu,\nu \in \mathcal{P}(\T^{d})$.}. 
\end{proof}

\begin{lem}\label{lem:Semi-continuity-SFI}
    Let $s\ge 1$, and assume that $V\in C^{1}(\T^{d})$. Then, $I_{s}(\cdot|\pi)$ is lower semicontinuous in $\mathcal{P}(\T^{d})$ with respect to the weak topology. Moreover, $I_{s}(\rho|\pi)=0$ if and only if $\rho=\pi$.
\end{lem}
\begin{proof}
    The first claim follows from the fact that $I_{s}(\cdot|\pi)^{1/2}$ can be written as the supremum of a family of weakly continuous functions defined in the space of probability measures $\mathcal{P}(\T^{d})$:
    \begin{equation*}
        I_{s}(\rho|\pi)^{1/2}=\lVert \pi\nabla \frac{\rho}{\pi}\rVert_{\dot H^{-s}}=\sup_{g\in C^{\infty},\, \widehat{g}_{0}=0,\, \lVert g\rVert_{\dot H^{s}}\le 1}\langle \pi\nabla \frac{\rho}{\pi},g\rangle= \sup_{g\in C^{\infty},\, \widehat{g}_{0}=0,\, \lVert g\rVert_{\dot H^{s}}\le 1}\langle \rho, \frac{1}{\pi}\diver(\pi g)\rangle,
    \end{equation*}
    where $\tfrac{1}{\pi}\diver (\pi g)$ is continuous due to the assumption $V\in C^{1}(\T^{d})$. To prove the second claim, observe that if $I_{s}(\rho|\pi)=0$, then $\pi\nabla \frac{\rho}{\pi}$ coincides with a constant vector $c\in \R^{d}$. Since $\pi>0$ and $\nabla \tfrac{\rho}{\pi}$ has zero average, this gives $c=0$, which in turn implies that $\tfrac{\rho}{\pi}$ is a constant. Since $\rho$ and $\pi$ have both unit mass, we deduce $\rho=\pi$.
\end{proof}

\begin{proof}[Proof of Proposition \ref{prop:qualitative-convergence-SVGD}]
    The proof follows from the abstract result in Proposition \ref{prop:abstract-lasalle}, with the choices $\mathcal{E}:= \mathcal{H}(\cdot|\pi)$ and $\mathcal{D}= I_{s}(\cdot|\pi)$, and $\delta=W_{1}$, taking into account Lemmas \ref{lem:motion-estimate-SVGD} and \ref{lem:Semi-continuity-SFI}.
\end{proof}

\section{Proof of the propagation of regularity estimates}\label{app:propagation-of-regularity}

In this section, we prove the propagation of regularity estimates stated in Section~\ref{sec:proof}. We recall their full statements, for convenience.

\begin{lem}
    Let $s\ge 1$ and $\gamma>\tfrac{d}{2}$. Let $V\in H^{\gamma+s}(\T^{d})$ and $\pi\in \mathcal{P}\cap H^{\gamma+s}(\T^{d})$ be defined as in \eqref{eq:def-pi-intermsof-V}. Let $\rho\in C_{t}\mathcal{P}_{x}\cap L_{t}^{\infty}H^{\gamma}_{x}$ be a solution of \eqref{eq:SVGD} defined up to some time $T>0$, and call $\sigma_{t}=\rho_{t}-\pi$. Then, there is a constant $C>0$ depending only on $d,s,\gamma$, and $\lVert V\rVert_{H^{\gamma+s}}$ such that the following estimate holds:
    \begin{equation}\label{eq:energy-estimate-higher-order-appendix}
        \frac{d}{dt}\lVert \sigma_{t}\rVert_{\dot H^{\gamma}}^{2}\le CL(t)\lVert \sigma_{t}\rVert_{\dot H^{\gamma}}^{2}-2\big(\min_{\T^{d}}\pi\big)\lVert \sigma_{t}\rVert_{\dot H^{\gamma-s+1}}^{2}+C\lVert \sigma_{t}\rVert_{\dot H^{\gamma-s}}\lVert \sigma_{t}\rVert_{\dot H^{\gamma-s+1}}\qquad \forall t\in (0,T),
    \end{equation}
    where we set
    \begin{equation}\label{eq:def-L-lipschitz-field-appendix}
        L(t):=\lVert \nabla^{2}D^{-2s}\sigma_{t}\rVert_{L^{\infty}}+\lVert \nabla D^{-2s}(\sigma_{t}\nabla V)\rVert_{L^{\infty}}.
    \end{equation}
\end{lem}

\begin{proof}
     Throughout the proof, we use the notation $D^{\beta}=(-\Delta)^{\beta/2}$ for all $\beta\in \R$. Using \eqref{eq:SVGD}, we see that $\sigma_{t}$ solves the equation
     \begin{equation*}
         \partial_{t}\sigma_{t}=\diver\left(\pi D^{-2s}(\pi\nabla\frac{\sigma_{t}}{\pi})\right)+ \diver\left(\sigma_{t} D^{-2s}(\pi\nabla\frac{\sigma_{t}}{\pi})\right),
     \end{equation*}
     where the first term on the right-hand side originates from the linearized equation, while the second contains the non-linear parts. 
     Therefore, we derive
     \begin{align*}
        \frac{d}{dt}\frac{\lVert \sigma_{t}\rVert_{\dot H^{\gamma}}^{2}}{2}&= \int_{\T^{d}} D^{\gamma}\partial_{t}\sigma_{t} D^{\gamma}\sigma_{t}\\
        &=\underbrace{\int_{\T^{d}} D^{\gamma}\diver\left(\pi D^{-2s}(\pi\nabla\frac{\sigma_{t}}{\pi})\right)D^{\gamma}\sigma_{t}}_{{\rm Lin}}+\underbrace{\int_{\T^{d}} D^{\gamma}\diver\left(\sigma_{t} D^{-2s}(\pi\nabla\frac{\sigma_{t}}{\pi})\right)D^{\gamma}\sigma_{t}}_{{\rm Nonlin}}.
    \end{align*}
    We will estimate separately the two integrals ${\rm Lin}$ and ${\rm Nonlin}$ above.

\smallskip
\noindent \textbf{Step 1: (The linear term).} Using $\pi \nabla \frac{\sigma_{t}}{\pi}= \nabla \sigma_{t}+\sigma_{t}\nabla V$, integration by parts, and self-adjointness of fractional derivatives, we may further split ${\rm Lin}$ into three terms:
\begin{align*}
    {\rm Lin}&= \int_{\T^{d}}D^{\gamma}\diver\left(\pi \nabla D^{-2s}\sigma_{t}\right)D^{\gamma}\sigma_{t}+\int_{\T^{d}}D^{\gamma}\diver\left(\pi D^{-2s}(\sigma_{t}\nabla V)\right)D^{\gamma}\sigma_{t}\\
    &=-\int_{\T^{d}}D^{\gamma+s}\left(\pi \nabla D^{-2s}\sigma_{t}\right)\cdot \nabla D^{\gamma-s}\sigma_{t}-\int_{\T^{d}}D^{\gamma+s}\left(\pi D^{-2s}(\sigma_{t}\nabla V)\right)\cdot \nabla D^{\gamma-s}\sigma_{t}\\
    &= -\int_{\T^{d}}\pi |\nabla D^{\gamma-s}\sigma_{t}|^{2}-\int_{\T^{d}}\big[D^{\gamma+s}\left(\pi \nabla D^{-2s}\sigma_{t}\right)-\pi D^{\gamma+s}\nabla D^{-2s}\sigma_{t}\big]\cdot \nabla D^{\gamma-s}\sigma_{t}\\
    &\quad\, -\int_{\T^{d}}D^{\gamma+s}\left(\pi D^{-2s}(\sigma_{t}\nabla V)\right)\cdot \nabla D^{\gamma-s}\sigma_{t}\\
    &= {\rm Lin}_{1}+ {\rm Lin}_{2}+ {\rm Lin}_{3}.
\end{align*}
Now, for ${\rm Lin}_{1}$, we simply bound $\pi$ below with its minimum value, and get
    \begin{equation}\label{eq:bound-Lin-1}
        {\rm Lin}_{1}\le -\big(\min_{\T^{d}}\pi\big)\int_{\T^{d}}|\nabla D^{\gamma-s}\sigma_{t}|^{2}=-\big(\min_{\T^{d}}\pi\big)\lVert \sigma_{t}\rVert_{\dot H^{\gamma-s+1}}^{2}.
    \end{equation}
For ${\rm Lin}_{2}$, using Cauchy-Schwarz, the Kato-Ponce commutator estimate \eqref{eq:kato-ponce}, and the Sobolev embedding $H^{\gamma+s-1}(\T^{d})\hookrightarrow L^{\infty}(\T^{d})$, we derive
    \begin{equation}\label{eq:bound-Lin-2}
        \begin{aligned}
            {\rm Lin}_{2}&\le \left\lVert D^{\gamma+s}\left(\pi\nabla D^{-2s}\sigma_{t}\right)-\pi D^{\gamma+s}\nabla D^{-2s}\sigma_{t}\right\rVert_{L^{2}} \lVert \nabla D^{\gamma-s}\sigma_{t}\rVert_{L^{2}}\\
        &\lesssim_{d,s}\left(\lVert \nabla \pi\rVert_{L^{\infty}}\lVert D^{\gamma+s-1}\nabla D^{-2s}\sigma_{t}\rVert_{L^{2}}+\lVert D^{\gamma+s}\pi\rVert_{L^{2}}\lVert \nabla D^{-2s}\sigma_{t}\rVert_{L^{\infty}}\right)\lVert \sigma_{t}\rVert_{\dot H^{\gamma-s+1}}\\
        &\lesssim_{d,s,\gamma, \lVert V\rVert_{H^{\gamma+s}}}\lVert \sigma_{t}\rVert_{\dot H^{\gamma-s}}\lVert \sigma_{t}\rVert_{\dot H^{\gamma-s+1}}.
        \end{aligned}
    \end{equation}
Finally, for ${\rm Lin}_{3}$, we use Cauchy-Schwarz and Lemma \ref{lem:equiv-homogeneous-norms-smooth-multiplier}:
\begin{equation}\label{eq:bound-Lin-3}
        {\rm Lin}_{3}\le \lVert D^{\gamma+s}\left(\pi D^{-2s}(\sigma_{t}\nabla V)\right)\rVert_{L^{2}}\lVert \nabla D^{\gamma-s}\sigma_{t}\rVert_{L^{2}}\lesssim_{d,s,\gamma,\lVert V\rVert_{H^{\gamma+s}}} \lVert \sigma_{t}\rVert_{\dot H^{\gamma-s}}\lVert \sigma_{t}\rVert_{\dot H^{\gamma-s+1}}.
\end{equation}

\smallskip
\noindent \textbf{Step 2: (The non-linear term).} After integrating by parts, we may split ${\rm Nonlin}$ into three terms:
    \begin{align*}
        {\rm Nonlin}&=-\int_{\T^{d}} D^{\gamma}\left(\sigma_{t}D^{-2s}(\pi\nabla \frac{\sigma_{t}}{\pi})\right)\cdot \nabla D^{\gamma}\sigma_{t}\\
        &=-\int_{\T^{d}} D^{\gamma}\sigma_{t} D^{-2s}(\pi\nabla \frac{\sigma_{t}}{\pi}) \cdot \nabla D^{\gamma}\sigma_{t}\\
        &\quad\,+ \int_{\T^{d}}\Big[D\left(D^{\gamma}\sigma_{t} D^{-2s}(\pi\nabla \frac{\sigma_{t}}{\pi})\right)-DD^{\gamma}\sigma_{t} D^{-2s}(\pi\nabla \frac{\sigma_{t}}{\pi})\Big]\cdot \nabla D^{\gamma-1}\sigma_{t}\\
        &\quad\, -\int_{\T^{d}}\Big[D^{\gamma+1}\left(\sigma_{t} D^{-2s}(\pi\nabla \frac{\sigma_{t}}{\pi})\right)-D^{\gamma+1}\sigma_{t} D^{-2s}(\pi\nabla \frac{\sigma_{t}}{\pi})\Big]\cdot \nabla D^{\gamma-1}\sigma_{t}\\
        &= {\rm Nonlin}_{1}+{\rm Nonlin}_{2}+{\rm Nonlin}_{3}.
    \end{align*}
    For ${\rm Nonlin}_{1}$, we use the identity $2D^{\gamma}\sigma_{t}\nabla D^{\gamma}\sigma_{t}=\nabla (D^{\gamma}\sigma_{t})^{2}$ together with integration by parts, and get
    \begin{equation}\label{eq:bound-Nonlin-1}
        {\rm Nonlin}_{1}=\frac{1}{2}\int_{\T^{d}} \diver\left(D^{-2s}(\pi\nabla \frac{\sigma_{t}}{\pi})\right)(D^{\gamma}\sigma_{t})^{2}\lesssim_{d} \lVert \nabla D^{-2s}(\pi\nabla \frac{\sigma_{t}}{\pi})\rVert_{L^{\infty}}\lVert \sigma_{t}\rVert_{\dot H^{\gamma}}^{2}.
    \end{equation}
    For ${\rm Nonlin}_{2}$, we use Cauchy-Schwarz together with \eqref{eq:kato-ponce-gamma-1}, and obtain
    \begin{equation}\label{eq:bound-Nonlin-2}
    \begin{aligned}
        {\rm Nonlin}_{2}&\le \lVert D\left(D^{\gamma}\sigma_{t} D^{-2s}(\pi\nabla \frac{\sigma_{t}}{\pi})\right)-DD^{\gamma}\sigma_{t} D^{-2s}(\pi\nabla \frac{\sigma_{t}}{\pi})\rVert_{L^{2}}\lVert \nabla D^{\gamma-1}\sigma_{t}\rVert_{L^{2}}\\
        &\lesssim_{d} \lVert \nabla D^{-2s}(\pi\nabla \frac{\sigma_{t}}{\pi})\rVert_{L^{\infty}}\lVert \sigma_{t}\rVert_{\dot H^{\gamma}}^{2}.
    \end{aligned}
    \end{equation}
    Finally, let us consider ${\rm Nonlin}_{3}$. Using Cauchy-Schwarz along with \eqref{eq:kato-ponce}, we get
    \begin{align*}
        {\rm Nonlin}_{3}&\le \lVert D^{\gamma+1}\left(\sigma_{t} D^{-2s}(\pi\nabla\frac{\sigma_{t}}{\pi})\right)-D^{\gamma+1}\sigma_{t} D^{-2s}(\pi\nabla\frac{\sigma_{t}}{\pi})\rVert_{L^{2}}\lVert \nabla D^{\gamma-1}\sigma_{t}\rVert_{L^{2}}\\
        &\lesssim_{d,s,\gamma} \left(\lVert \nabla D^{-2s}(\pi \nabla \frac{\sigma_{t}}{\pi})\rVert_{L^{\infty}}\lVert D^{\gamma}\sigma_{t}\rVert_{L^{2}}+\lVert \sigma_{t}\rVert_{L^{p}}\lVert D^{\gamma+1}D^{-2s}(\pi \nabla \frac{\sigma_{t}}{\pi})\rVert_{L^{q}}\right)\lVert \sigma_{t}\rVert_{\dot H^{\gamma}},
    \end{align*}
    where $p= 2\frac{\gamma+2s-2}{2s-2}$ and $q=2\frac{\gamma+2s-2}{\gamma}$ are such that $\frac{1}{p}+\frac{1}{q}=\frac{1}{2}$. By the fractional Gagliardo-Nirenberg inequality \cite{brezis2018gagliardo}, we may bound
    \begin{align*}
        &\lVert \sigma_{t}\rVert_{L^{p}}\lVert D^{\gamma+1}D^{-2s}(\pi \nabla \frac{\sigma_{t}}{\pi})\rVert_{L^{q}}\lesssim_{d,s,\gamma}\\&\lesssim_{d,s,\gamma} \lVert \nabla^{2}D^{-2s}\sigma_{t}\rVert_{L^{\infty}}^{\frac{2}{q}}\lVert D^{\gamma}\sigma_{t}\rVert_{L^{2}}^{\frac{2}{p}}\lVert \nabla D^{-2s}(\pi \nabla \frac{\sigma_{t}}{\pi})\rVert_{L^{\infty}}^{\frac{2}{p}}\lVert D^{\gamma-1}(\pi \nabla \frac{\sigma_{t}}{\pi})\rVert_{L^{2}}^{\frac{2}{q}}\\
        &\lesssim_{d,s,\gamma, \lVert V\rVert_{H^{\gamma+s}}} \left(\lVert \nabla^{2}D^{-2s}\sigma_{t}\rVert_{L^{\infty}}+\lVert \nabla D^{-2s}(\pi \nabla \frac{\sigma_{t}}{\pi})\rVert_{L^{\infty}}\right)\lVert \sigma_{t}\rVert_{\dot H^{\gamma}},
    \end{align*}
    where in the last step we used Young inequality, and Lemma \ref{lem:equiv-homogeneous-norms-smooth-multiplier}. As a consequence, we find
    \begin{equation}\label{eq:bound-Nonlin-3}
        {\rm Nonlin}_{3}\lesssim_{d,s,\gamma,\lVert V\rVert_{H^{\gamma+s}}} \left(\lVert \nabla^{2}D^{-2s}\sigma_{t}\rVert_{L^{\infty}}+\lVert \nabla D^{-2s}(\pi \nabla \frac{\sigma_{t}}{\pi})\rVert_{L^{\infty}}\right)\lVert \sigma_{t}\rVert_{\dot H^{\gamma}}^{2}.
    \end{equation}

    The energy estimate \eqref{eq:energy-estimate-higher-order-appendix} is obtained combining \eqref{eq:bound-Lin-1},\eqref{eq:bound-Lin-2},\eqref{eq:bound-Lin-3},\eqref{eq:bound-Nonlin-1},\eqref{eq:bound-Nonlin-2}, and \eqref{eq:bound-Nonlin-3}.
\end{proof}

It remains to control the cumulative effect of the Lipschitz norm appearing in the high-order estimate.

\begin{lem}\label{lem:time-integral-Lipschitz-field-app}
    Under the same assumptions and using the same notation of Lemma \ref{lem:energy-estimate-higher-order}, assume further that $\gamma>s-1$ and $\lVert \sigma_{t}\rVert_{\dot H^{\gamma}}^{2}\le M$ for all $t\in [0,T)$. Then, there is $\beta=\beta(d,s,\gamma)>0$ such that 
    \begin{equation*}
        \int_{0}^{T}L(t)\, dt\lesssim_{d,s,\gamma,\lVert V\rVert_{H^{\gamma+s}}, M} \mathcal{H}(\bar\rho|\pi)^{\beta}.
    \end{equation*}
\end{lem}
\begin{proof} 
    Let $r:=\max \{2-2s+\tfrac{\gamma}{2}+\tfrac{d}{4}, 1-s\}$. Since $\gamma>\tfrac{d}{2}$, we have $r>2-2s+\tfrac{d}{2}$. Therefore, by the Sobolev embedding and Lemma \ref{lem:equiv-homogeneous-norms-smooth-multiplier}, we may bound 
    \begin{equation*}
       L(t)=\lVert \nabla^{2}D^{-2s}\sigma_{t}\rVert_{L^{\infty}}+\lVert \nabla D^{-2s}(\sigma_{t}\nabla V)\rVert_{L^{\infty}}\lesssim_{d,\gamma,s,\lVert V\rVert_{H^{\gamma+s}}} \lVert \sigma_{t}\rVert_{\dot H^{r}}.
    \end{equation*}

    Notice that we have $r\in [1-s,\gamma)$, hence we may use Sobolev interpolation of homogeneous norms to get
    \begin{equation*}
        \lVert \sigma_{t}\rVert_{\dot H^{r}}\le \lVert \sigma_{t}\rVert_{\dot H^{1-s}}^{1-\theta}\lVert \sigma_{t}\rVert_{\dot H^{\gamma}}^{\theta}\qquad \theta:=\frac{r+s-1}{\gamma+s-1}.
    \end{equation*}
    Then, using \eqref{eq:approx-formula-SFI} and \eqref{eq:coercivity-interpolation-SFI} together with $\lVert \sigma_{t}\rVert_{\dot H^{\gamma}}^{2}\le M$, we find
    \begin{align*}
       L(t)&\lesssim_{d,s,\gamma,\lVert V\rVert_{H^{\gamma+s}},M} I_{s}(\rho_{t}|\pi)I_{s}(\rho_{t}|\pi)^{-\frac{1+\theta}{2}}\\ &\lesssim_{d,s,\gamma,M} \left(-\frac{d}{dt}\mathcal{H}(\rho_{t}|\pi)\right)\mathcal{H}(\rho_{t}|\pi)^{-\frac{1+\theta}{2}\left(1+\frac{s-1}{\gamma}\right)}=-\frac{1}{\beta}\frac{d}{dt}\mathcal{H}(\rho_{t}|\pi)^{\beta},
    \end{align*}
    where $\beta= \min\{\tfrac{2\gamma-d}{8\gamma}, \tfrac{\gamma-s+1}{2\gamma}\}>0$. The result follows by integration. 
\end{proof}

\section{Maximum principle for $s=1$}\label{app:maximum-principle}

In the endpoint case $s=1$, the equation has an additional maximum-principle structure, which gives global control of bounded solutions. This allows us to ensure that the equation is globally well-posed.

\begin{lem}[Maximum principle for $s=1$]\label{lem:maximum-principle-s=1}  Let $s=1$, and $\gamma>\tfrac{d}{2}$, $\gamma\ge 1$. Let $V\in H^{\gamma}(\T^{d})$ and $\pi \in \mathcal{P}\cap H^{\gamma}(\T^{d})$ be defined as in \eqref{eq:def-pi-intermsof-V}. Let $\bar\rho \in \mathcal{P}\cap L^{\infty}(\T^{d})$ and  $\rho\in C_{t}\mathcal{P}_{x}\cap  L^{\infty}_{t,x}$ be a solution of \eqref{eq:SVGD}. Then $\rho$ can be extended to all positive times, and there is a constant $M>0$ depending only on $d,\gamma$, $\lVert V\rVert_{H^{\gamma}}$, and $\mathcal{H}(\bar\rho|\pi)$ such that 
\begin{equation}\label{eq:maximum-principle-s=1}
    \lVert \rho_{t}\rVert_{L^{\infty}}\le \max\{M,\lVert \bar\rho\rVert_{L^{\infty}}\}\qquad \forall t\ge 0.
\end{equation}
\end{lem}

\begin{proof}
     Arguing by approximation, we may assume that the solution is smooth. We call $\sigma_{t}=\rho_{t}-\pi$ and $v_{t}=-\nabla K_{1}*\sigma_{t}-K_{1}*(\sigma_{t}\nabla V)$ the velocity field. The equation can be rewritten as
    \begin{equation*}
        \partial_{t}\rho_{t}+\nabla \rho_{t} \cdot v_{t}=-\rho_{t}\diver v_{t}=-\rho_{t}\left(\sigma_{t}-\diver(-\Delta)^{-1}(\sigma_{t}\nabla V)\right).
    \end{equation*}
    Let $r>d$ be such that the embedding $H^{\gamma-1}(\T^{d})\hookrightarrow L^{r}(\T^{d})$ holds (the existence of such $r$ is granted by the assumptions on $\gamma$). Let $p\in (d,r)$ and $q\in (p,\infty)$ be such that $\tfrac{1}{p}=\tfrac{1}{r}+\tfrac{1}{q}$. Here $p,r,q$ can be chosen depending only on $d$ and $\gamma$. Then, by Calder\'on-Zygmund estimates, Sobolev embedding, and H\"{o}lder's inequality, we find
    \begin{align*}
        \lVert\diver(-\Delta)^{-1}(\sigma_{t}\nabla V)\rVert_{L^{\infty}}&\lesssim_{d,\gamma} \lVert \sigma_{t} \nabla V\rVert_{L^{p}}\le \lVert \nabla V\rVert_{L^{r}}\lVert \sigma_{t}\rVert_{L^{q}}\\
        &\lesssim_{d,\gamma,\lVert V\rVert_{H^{\gamma}}} \lVert \sigma_{t}\rVert_{L^{1}}^{\tfrac{1}{q}}\lVert \sigma_{t}\rVert_{L^{\infty}}^{\tfrac{q-1}{q}}\\
        &\lesssim_{d,\gamma} \mathcal{H}(\rho_{t}|\pi)^{\tfrac{1}{2q}}\lVert \sigma_{t}\rVert_{L^{\infty}}^{\tfrac{q-1}{q}}\le \mathcal{H}(\bar\rho|\pi)^{\tfrac{1}{2q}}\left(\max_{\T^{d}}\rho_{t}+\max_{\T^{d}}\pi\right)^{\tfrac{q-1}{q}},
    \end{align*}
    where we also used Pinsker's inequality $\lVert \rho_{t}-\pi\rVert_{L^{1}}\le \sqrt{2\mathcal{H}(\rho_{t}|\pi)}$ and the fact that the relative entropy is dissipated by \eqref{eq:stein-dissipation-identity}. For $t>0$, if $x_{0}\in \T^{d}$ is a maximum point of $\rho_{t}$, then $\nabla \rho_{t}(x_{0})=0$ and 
    \begin{align*}
        \partial_{t}\rho_{t}(x_{0})&=-\rho_{t}(x_{0})\left(\rho_{t}(x_{0})-\pi(x_{0})-\diver(-\Delta)^{-1}(\sigma_{t}\nabla V)(x_{0})\right)\\
        &\le -\max_{\T^{d}}\rho_{t}\left(\max_{\T^{d}}\rho_{t}-\max_{\T^{d}}\pi-C\left(\max_{\T^{d}}\rho_{t}+\max_{\T^{d}}\pi\right)^{\tfrac{q-1}{q}}\right).
    \end{align*}
    The expression in parentheses on the right-hand side is positive if $\max_{\T^{d}}\rho_{t}\ge M$, with $M$ depending only on $d,\gamma$, $\lVert V\rVert_{H^{\gamma}}$, and $\mathcal{H}(\bar\rho|\pi)$. Since for almost every $t>0$, $\tfrac{d}{dt}\max_{\T^{d}}\rho_{t}=\partial_{t}\rho_{t}(x_{0})$, with $x_{0}$ maximum point of $\rho_{t}$, \eqref{eq:maximum-principle-s=1} follows.
\end{proof}

\section{More general kernels}\label{app:general-kernel}
Here we describe how to change the proof of quantitative convergence in the case where $K_{s}$ is replaced by a more general kernel $K$ with the same spectral behavior.

\begin{thm}[General kernels]\label{thm:general_kernels}
    Let $s>1$, $K:\T^{d}\to \R$ be a symmetric kernel and $\mathcal{K}(f):=K*f$ the corresponding convolution operator. Assume that the following spectral conditions hold:
    \begin{itemize}
        \item [(i)] $\widehat{K}_{0}\ge 0$ and $\widehat{K}_{k}\approx |k|^{-2s}$ for every $k\in \Z^{d}\setminus \{0\}$.
        \item [(ii)] The operator $T:=(-\Delta)^{s}\mathcal{K}$ satisfies the following estimates:
        \begin{equation*}
            \lVert Tf\rVert_{L^{p}}\approx_{p} \lVert f\rVert_{L^{p}}\qquad \forall f\in L^{p}(\T^{d}),\quad \widehat{f}_{0}=0,\qquad \forall p\in (1,\infty).
        \end{equation*}
        \item [(iii)] There is a constant $\bar c>0$ such that
        \begin{equation*}
            \lVert \left(\mathcal{K}-\bar c (-\Delta)^{-s}\right)(f)\rVert_{L^{2}}\lesssim \lVert f\rVert_{\dot H^{-2s-1}}\qquad \forall f\in L^{2}(\T^{d}),\quad \widehat{f}_{0}=0.
        \end{equation*}
    \end{itemize}
    Then the same conclusions of Theorem \ref{thm:quantitative-convergence-s>1} hold true for solutions of \eqref{eq:SVGD} if we replace the Riesz kernel $K_{s}$ with $K$.
\end{thm}

\begin{rmk}\label{rmk:Askey-kernel}
    This theorem can be applied, in particular, to the negative distance-like kernel $K$ introduced in \eqref{eq:negative-distance-kernel}, for which $s=(d+1)/2$. Indeed, since $K$ is supported in $B_{1/4}$, it may equivalently be regarded as a function on $\R^{d}$. We can then study the relevant properties of its Fourier transform $\widehat{K}$, which is real-valued, radial, and smooth. To prove points $(i)$ and $(iii)$, one may first apply Askey's Theorem \cite{Askey1973RadialCharacteristicFunctions} to get $\widehat{K}>0$. The asymptotic behavior is then determined by the $1$-homogeneous singularity of the kernel at the origin, yielding $\widehat{K}(\xi)= C_{d}|\xi|^{-d-1}+O(|\xi|^{-d-2})$ as $\xi \to \infty$, which gives the desired spectral comparability. Finally, point $(ii)$ follows by applying the standard Mikhlin multiplier Theorem to the symbol $c(\xi):=|\xi|^{d+1}\widehat{K}(\xi)$, once one proves that $|\partial^{\beta}c(\xi)|\lesssim |\xi|^{-|\beta|}$, for all $|\beta|\le \lfloor \tfrac{d}{2}\rfloor +1$. These estimates hold thanks to the regularity imposed to the kernel away from the origin by raising $(1-4|x|)$ to the sufficiently large power $d+2$. 
\end{rmk}

In order to prove Theorem \ref{thm:general_kernels}, one can proceed as in the proof of Theorem \ref{thm:quantitative-convergence-s>1}, with some localized differences that we highlight in the following lines.

\smallskip
\noindent \textit{The equation.} Let $\mathcal{K}=K*$ be the convolution operator associated to the kernel $K$. We consider a solution of 
\begin{equation*}
    \partial_{t}\rho=\diver\left(\rho \mathcal{K}(\rho \nabla\log\left(\frac{\rho}{\pi}\right))\right),\qquad \rho_{0}=\bar\rho.
\end{equation*}
Denoting $\sigma_{t}=\rho_{t}-\pi$, the velocity field can be rewritten in the following equivalent ways:
\begin{equation*}
    v_{t}:=-\mathcal{K}(\rho_{t} \nabla\log\left(\frac{\rho_{t}}{\pi}\right))=-\mathcal{K}(\pi \nabla \frac{\sigma_{t}}{\pi})=-\nabla \mathcal{K}(\sigma_{t})-\mathcal{K}(\sigma_{t}\nabla V).
\end{equation*}

\smallskip
\noindent \textit{Energy dissipation identity \eqref{eq:stein-dissipation-identity}.} We get 
    \begin{equation}\label{eq:stein-dissipation-identity-general-kernels}
        \frac{d}{dt}\mathcal{H}(\rho_{t}|\pi)=-\langle \mathcal{K}(\pi \nabla\frac{\sigma_{t}}{\pi}), \pi \nabla\frac{\sigma_{t}}{\pi}\rangle =:-I(\rho_{t}|\pi).
    \end{equation}

\smallskip
\noindent \textit{Lemma \ref{lem:comparability-stein-H1-s}.} By assumption $(i)$, the following equivalence holds:
    \begin{equation*}
        I(\rho_{t}|\pi)\approx \widehat{K}_{0}\left(\widehat{(\sigma_{t}\nabla V)}_{0}\right)^{2}+\lVert \pi \nabla \frac{\sigma_{t}}{\pi}\rVert_{\dot H^{-s}}^{2}.
    \end{equation*}
    Therefore, since $\widehat{K}_{0}\ge 0$, we deduce 
    \begin{equation*}
        I(\rho_{t}|\pi)\gtrsim \lVert \sigma_{t}\rVert_{\dot H^{1-s}}^{2}\gtrsim \mathcal{H}(\rho_t|\pi)^{1+\frac{s-1}{\gamma}}\left(\lVert \sigma_{t}\rVert_{\dot H^{\gamma}}^{2}\right)^{-\frac{s-1}{\gamma}}.
    \end{equation*}

\smallskip
\noindent\textit{Lemma \ref{lem:energy-estimate-higher-order}.} We may prove the following estimate, for a sufficiently large constant $C>0$:
    \begin{equation}\label{eq:energy-estimate-higher-order-general-kernels}
        \frac{d}{dt}\lVert \sigma_{t}\rVert_{\dot H^{\gamma}}^{2}\le CL(t)\lVert \sigma_{t}\rVert_{\dot H^{\gamma}}^{2}-\frac{1}{C}\lVert \sigma_{t}\rVert_{\dot H^{\gamma-s+1}}^{2}+C\lVert \sigma_{t}\rVert_{\dot H^{\gamma-s}}\lVert \sigma_{t}\rVert_{\dot H^{\gamma-s+1}},
    \end{equation}
    where we set
    \begin{equation}\label{eq:def-L-lipschitz-field-general-kernels}
        L(t):=\lVert \nabla^{2}\mathcal{K}(\sigma_{t})\rVert_{L^{\infty}}+\lVert \nabla \mathcal{K}(\sigma_{t}\nabla V)\rVert_{L^{\infty}}.
    \end{equation}
    The proof is almost identical to that of Lemma~\ref{lem:energy-estimate-higher-order}, after replacing $D^{-2s}$ with $\mathcal{K}$, and using their spectral comparability on zero-mean functions. The only notable difference appears in the estimate of the term ${\rm Lin}$ arising from the linearized equation. There, the negative term ${\rm Lin}_{1}$ is now replaced by 
    \begin{align*}
        -\int_{\T^{d}}\pi D^{\gamma+s}\nabla \mathcal{K}(\sigma_{t})\cdot \nabla D^{\gamma-s}\sigma_{t}&=-\bar c\int_{\T^{d}}\pi |\nabla D^{\gamma-s}\sigma_{t}|^{2}\\
        &\quad\,-\int_{\T^{d}}\pi D^{\gamma+s}\nabla\left(\mathcal{K}-\bar c D^{-2s}\right)(\sigma_{t})\cdot \nabla D^{\gamma-s}\sigma_{t}\\
        &=:{\rm Lin}_{1,1}+{\rm Lin}_{1,2},
    \end{align*}
    where we added and subtracted $\bar c D^{-2s}$ to $\mathcal{K}$, with $\bar c>0$ the constant given by assumption $(iii)$. Then, ${\rm Lin}_{1,1}$ is estimated by
    \begin{equation*}
        {\rm Lin}_{1,1}\le -\bar c \big(\min_{\T^{d}}\pi\big)\int_{\T^{d}}|\nabla D^{\gamma-s}\sigma_{t}|^{2}=-\bar c \big(\min_{\T^{d}}\pi\big)\lVert \sigma_{t}\rVert_{\dot H^{\gamma-s+1}}^{2}.
    \end{equation*}
    On the other hand, since $\mathcal{K}-\bar cD^{-2s}$ gains one spectral order by assumption $(iii)$, we may bound ${\rm Lin}_{1,2}$ by
    \begin{equation*}
        {\rm Lin}_{1,2}\lesssim \lVert \sigma_{t}\rVert_{\dot H^{\gamma-s}}\lVert \sigma_{t}\rVert_{\dot H^{\gamma-s+1}}.
    \end{equation*}

    \smallskip
    \noindent \textit{Lemma \ref{lem:time-integral-Lipschitz-field}.} The result is exactly the same using $L$ from \eqref{eq:def-L-lipschitz-field-general-kernels}. Indeed, by assumption $(i)$, the estimate $L(t)\lesssim \lVert \sigma_{t}\rVert_{\dot H^{r}}$ still holds for the same exponent $r$ introduced in the proof of Lemma \ref{lem:time-integral-Lipschitz-field}.

    \smallskip
    \noindent \textit{Conclusion.} Having derived the same ingredients for the case of more general kernels, the proof of Theorem \ref{thm:general_kernels} follows exactly as that of Theorem \ref{thm:quantitative-convergence-s>1}.

\smallskip
\paragraph{Acknowledgements.} M.C., R.C., and X.F.~are supported by the Swiss State Secretariat for Education, Research and Innovation (SERI) under contract number MB22.00034 through the project TENSE. R.C. and X.F. ~are also supported by the Swiss National Science Foundation (SNF grant PZ00P2\_208930). X.F. is further supported  by the AEI project PID2021-125021NA-I00 (Spain).

\bibliographystyle{alpha}
\bibliography{Refs_SteinGF}

\end{document}